\pdfoutput=1 
\documentclass[accepted]{uai2026} 

\usepackage[american]{babel}

\usepackage{natbib} 
\usepackage{mathtools} 
\usepackage{booktabs} 
\usepackage{longtable} 

\usepackage{amsfonts}
\usepackage{amssymb}
\usepackage{amsthm}
\usepackage{graphicx}
\usepackage[table]{xcolor}
\usepackage{etoolbox}
\usepackage{mdframed}
\usepackage{tikz}
\usetikzlibrary{arrows.meta, positioning, shapes.geometric, calc}
\usepackage{multirow}
\usepackage{algorithm}
\usepackage{algpseudocode}
\usepackage{pifont}
\usepackage{cleveref}
\usepackage{titletoc}

\renewcommand{\eqref}[1]{Eq.~\ref{#1}} 

\graphicspath{{figures_new/}{figures/}}

\definecolor{theoremgray}{gray}{0.95}
\definecolor{tablegray}{gray}{0.96}
\definecolor{tableheadgray}{gray}{0.90}
\definecolor{tablegroupgray}{gray}{0.94}
\newcommand{\tablestyle}{%
	\rowcolors{2}{tablegray}{white}%
	\arrayrulecolor{black}%
	\renewcommand{\arraystretch}{1.12}%
	\setlength{\tabcolsep}{4.5pt}%
	\setlength{\aboverulesep}{0pt}%
	\setlength{\belowrulesep}{0pt}%
}
\AtBeginEnvironment{tabular}{\tablestyle}
\newcommand{\thc}[1]{#1}
\newcommand{\theadrow}[1]{\rowcolor{tableheadgray}#1\\}
\newcommand{\tablegroup}[2]{\rowcolor{tablegroupgray}\multicolumn{#1}{c}{#2}}
\mdfdefinestyle{theorembox}{
	backgroundcolor=theoremgray,
	hidealllines=true,
	linewidth=0pt,
	innerleftmargin=0pt,
	innerrightmargin=0pt,
	innertopmargin=0pt,
	innerbottommargin=0pt,
	skipabove=4pt,
	skipbelow=4pt
}

\theoremstyle{plain}
\newtheorem{theorem}{Theorem}

\newtheorem{proposition}[theorem]{Proposition}
\newtheorem{corollary}[theorem]{Corollary}
\theoremstyle{definition}
\newtheorem{definition}{Definition}
\theoremstyle{remark}
\newtheorem{remark}{Remark}
\theoremstyle{plain}
\newtheorem{assumption}{Assumption}

\surroundwithmdframed[style=theorembox]{theorem}
\surroundwithmdframed[style=theorembox]{lemma}
\surroundwithmdframed[style=theorembox]{proposition}
\surroundwithmdframed[style=theorembox]{corollary}
\surroundwithmdframed[style=theorembox]{definition}
\surroundwithmdframed[style=theorembox]{remark}
\surroundwithmdframed[style=theorembox]{assumption}

\newcommand{\norm}[1]{\left\|#1\right\|}
\newcommand{\abs}[1]{\left|#1\right|}
\newcommand{\E}{\mathbb{E}}

\newcommand{\R}{\mathbb{R}}
\newcommand{\C}{\mathbb{C}}

\newcommand{\Tr}{\mathrm{Tr}}

\DeclareMathOperator*{\argmin}{argmin}

\DeclareMathOperator{\spec}{spec}

\DeclareMathOperator{\softplus}{softplus}

\title{Residual Koopman Spectral Profiling for Predicting and Preventing Transformer Training Instability}

\author[1]{Bum Jun Kim}
\author[1]{Shohei Taniguchi}
\author[1]{Makoto Kawano}
\author[1]{Yusuke Iwasawa}
\author[1]{Yutaka Matsuo}
\affil[1]{Graduate School of Engineering, The University of Tokyo, Japan}

\begin{document}

\maketitle

\begin{abstract}
	Training divergence in transformers wastes compute, yet practitioners discover instability only after expensive runs begin. They therefore need an expected probability of failure for a transformer before training starts. Our study of Residual Koopman Spectral Profiling (RKSP) provides such an estimate. From a single forward pass at initialization, RKSP extracts Koopman spectral features by applying whitened dynamic mode decomposition to layer-wise residual snapshots. Our central diagnostic, the near-unit spectral mass, quantifies the fraction of modes concentrated near the unit circle, which captures instability risk. For predicting divergence across extensive configurations, this estimator achieves an AUROC of 0.995, outperforming the best gradient baseline. We further make this diagnostic actionable through Koopman Spectral Shaping (KSS), which reshapes spectra during training. We empirically validate that our method works in practice: RKSP predicts divergence at initialization, and when RKSP flags high risk, turning on KSS successfully prevents divergence. In the challenging high learning rate regime without normalization layers, KSS reduces the divergence rate from 66.7\% to 12.5\% and enables learning rates that are 50\% to 150\% higher. These findings generalize to WikiText-103 language modeling, vision transformers on CIFAR-10, and pretrained language models, including GPT-2 and LLaMA-2 up to 7B, as well as emerging architectures such as MoE, Mamba-style SSMs, and KAN.
\end{abstract}

\section{Introduction}
\label{sec:intro}
Transformers \citep{vaswani2017attention} exhibit unpredictable training dynamics despite stabilization techniques such as layer normalization \citep{ba2016layernorm} and careful initialization. In particular, exploding-gradient instability remains a known failure mode in deep networks \citep{pascanu2013difficulty}. This unpredictability in training divergence wastes compute: practitioners often discover divergence only after expensive runs have begun. Addressing this problem requires a calibrated, initialization-time risk estimate, which enables early pruning of risky configurations. Such an estimate should output a probability that matches empirical frequency rather than a heuristic score.

To provide such estimates, we propose Residual Koopman Spectral Profiling (RKSP), which views transformer layers as discrete-time dynamical systems. From a single forward pass, RKSP applies whitened dynamic mode decomposition (DMD) \citep{tu2014dmd} to estimate local linear operators that approximate the residual stream evolution $\mathbf{h}_0 \to \mathbf{h}_1 \to \cdots \to \mathbf{h}_L$. The resulting layer-wise spectra provide a compact, predictive summary of stability.

Intuitively, when a layer's local linearization is close to normal, eigenvalues near the unit circle imply near-isometric propagation and therefore weak damping of signals and gradients. In high learning rate regimes, such weak damping can leave perturbations and optimization noise to be less attenuated, increasing divergence risk; conversely, strongly contractive spectra tend to be more stable but may cause rapid gradient decay. We summarize this trade-off with the near-unit mass $M_{\approx 1}$, the fraction of modes near the unit circle, together with a separate measure of non-normality; Section~\ref{sec:theory} provides a theoretical explanation for this behavior.

The near-unit mass $M_{\approx 1}$ exhibits a strong instability signal; using the monotone risk score $M_{\approx 1}$ yields an Area Under the Receiver Operating Characteristic curve (AUROC) of 0.995 for divergence prediction across various normalization strategies, tasks, and architectures. To make this signal further actionable, we introduce Koopman Spectral Shaping (KSS), which reshapes spectra during training to reduce divergence.

\noindent Our contributions are as follows.
\begin{itemize}
	\item We propose RKSP, which estimates layer-wise Koopman spectra at initialization via whitened DMD (Section~\ref{sec:rsp}).
	\item We propose KSS, which reshapes spectra during training, preventing instability and permitting larger learning rates (Section~\ref{sec:kss}).
	\item We provide theoretical bounds linking near-unit mass, near-normality, and the trade-offs between instability and expressivity (Section~\ref{sec:theory}).
	\item We validate RKSP across various normalization strategies, datasets, and architectures, including language models, vision transformer (ViT) models, and emerging neural network families (Section~\ref{sec:experiments} and Appendix~\ref{app:scale}).
\end{itemize}

\section{Background}
\label{sec:background}
A transformer with $L$ layers updates its residual stream according to the residual formulation \citep{he2016resnet}:
\begin{equation}
	\mathbf{h}_{\ell+1} = \mathbf{h}_\ell + f_\ell(\mathbf{h}_\ell; \theta_\ell) = F_\ell(\mathbf{h}_\ell),
	\label{eq:residual}
\end{equation}
where $f_\ell$ comprises self-attention or multi-layer perceptron sub-layers. This formulation reveals that each layer transition defines a discrete-time dynamical system. Also, residual networks can be viewed as discretizations of continuous-time dynamical systems, motivating stability analysis from an ordinary differential equation perspective \citep{haber2017stable,chen2018neuralode}. Indeed, the local linearization of this system characterizes its stability.

\subsection{Koopman Operator Theory and DMD Primer}
Consider a discrete-time dynamical system $\mathbf{x}_{t+1} = F(\mathbf{x}_t)$ on state space $\mathcal{X} \subseteq \R^d$. The Koopman operator $\mathcal{K}: \mathcal{F} \to \mathcal{F}$ acts on observable functions $g: \mathcal{X} \to \C$ via composition with the dynamics \citep{koopman1931hamiltonian,mezic2005spectral,mezic2013analysis,tu2014dmd}:
\begin{equation}
	(\mathcal{K}g)(\mathbf{x}) \triangleq g(F(\mathbf{x})).
	\label{eq:koopman_def}
\end{equation}
For a nonlinear $F$, the operator $\mathcal{K}$ is infinite-dimensional but remains linear regardless of the nonlinearity in $F$. The spectral properties of $\mathcal{K}$---its eigenvalues $\{\lambda_j\}$ and eigenfunctions $\{\phi_j\}$---encode the intrinsic timescales and geometric structure of the dynamics:
\begin{equation}
	\begin{aligned}
		\mathcal{K}\phi_j                         & = \lambda_j \phi_j,                                          \\
		\text{therefore } \phi_j(F^n(\mathbf{x})) & = \lambda_j^n \phi_j(\mathbf{x}),  \text{for all } n \geq 0.
	\end{aligned}
	\label{eq:koopman_mode}
\end{equation}
The spectrum admits direct interpretation: modes with $|\lambda_j| > 1$ grow exponentially, modes with $|\lambda_j| < 1$ decay exponentially, and modes with $|\lambda_j| = 1$ persist or oscillate. The argument $\arg(\lambda_j)$ gives the oscillation frequency of mode $j$ \citep{rowley2009spectral}.

In a layer-wise, non-autonomous setting, each layer has its own distinct operator, so each $\hat{\mathbf{A}}_\ell$ must be estimated separately. Thus, the modulus $|\lambda(\hat{\mathbf{A}}_\ell)|$ indicates a local, per-layer expansion or contraction tendency. DMD approximates these layer-wise Koopman operators from finite data \citep{schmid2010dmd,tu2014dmd,kutz2016dmdbook}.

\subsection{Related Work}
\label{sec:related}
\paragraph{Koopman methods in machine learning.} Existing Koopman methods approximate operators via DMD and its extensions or learn linearizing transformations for dynamics prediction \citep{lusch2018deeplin,williams2015edmd,korda2018linear}. These methods primarily target representation learning.

\paragraph{Edge of chaos.} The edge of chaos hypothesis links optimal trainability to critical initialization and signal propagation regimes \citep{schoenholz2017deepinfo,poole2016chaos,pennington2017dynamical}, where signals neither explode nor vanish.

\paragraph{Neural tangent kernel.} The neural tangent kernel characterizes gradient descent in the infinite-width limit and yields kernel-regression behavior \citep{jacot2018ntk}.
In particular, linearizing the network around its initialization makes the kernel essentially constant, so training reduces to regression with this fixed kernel.

\paragraph{Mean-field theory and $\mu$P.} Mean-field theory and Maximal Update Parameterization ($\mu$P) enable hyperparameter transfer across scales through asymptotic analysis \citep{schoenholz2017deepinfo,yang2022tensorprograms}.

\paragraph{Normalization-free training.} Normalization-free residual networks can be stabilized with Fixup initialization \citep{zhang2019fixup}. ReZero trains deep residual networks and transformers without normalization by introducing a residual scaling parameter initialized to zero, so the network starts near an identity map \citep{bachlechner2020rezero}. Normalizer-Free Networks replace normalization with scaled activations and adaptive gradient clipping, enabling stable large-scale training without normalization layers \citep{brock2021nfnet}. We discuss Fixup and ReZero-style identity initialization through the RKSP lens in Appendix~\ref{app:practical_notes}.

\paragraph{Positioning of this study.} RKSP uses Koopman spectra as an initialization-time diagnostic: from a single forward pass, we estimate layer-wise operators and predict divergence risk. Unlike Koopman representation learning for forecasting \citep{lusch2018deeplin,williams2015edmd,korda2018linear}, our near-unit mass $M_{\approx 1}$ provides a measurable handle on critical signal propagation \citep{schoenholz2017deepinfo,poole2016chaos,pennington2017dynamical} and captures instabilities beyond neural tangent kernel analyses \citep{jacot2018ntk}; we make it actionable via KSS and validate it in no-normalization regimes \citep{zhang2019fixup,brock2021nfnet}.

\section{Method}
\label{sec:method}

\subsection{Problem Setup}
\label{sec:setup}

\paragraph{Divergence Definition} A training run is marked as diverged if, at any step, the loss exceeds 50.0 or the gradient norm exceeds 500.0. This criterion defines the binary label $D \in \{0,1\}$ used throughout our experiments.

\paragraph{Divergence Prediction Task} Given a model architecture, normalization strategy, optimizer choice, and dataset, we collect $N$ residual-stream snapshots at initialization from a single forward pass and compute the spectral profile $\mathcal{S}$. A probabilistic predictor then maps $\mathcal{S}$ to $P(D=1 \mid \mathcal{S})$, estimating divergence risk before training begins.

\subsection{Residual Koopman Spectral Profiling}
\label{sec:rsp}
Algorithm~\ref{alg:rsp} summarizes our RKSP procedure. The algorithm computes DMD for each layer to obtain its spectral profile: $\rho_\ell$ denotes the spectral radius of $\hat{\mathbf{A}}_\ell$, $\kappa_\ell$ the eigenvector condition number of its eigenbasis, and $\mathcal{K}_\ell$ the Kreiss constant (Appendix~\ref{app:theory}).
\begin{algorithm}[t]
	\caption{Residual Koopman Spectral Profiling}
	\label{alg:rsp}
	\begin{algorithmic}[1]
		\Require Model $\mathcal{M}$ with $L$ layers; dataset $\mathcal{D}$; the number of samples $N$
		\Ensure Spectral profile $\mathcal{S} = \{(M_{\approx 1}^\ell, \rho_\ell, \kappa_\ell, \eta_{\mathrm{nl}}^\ell, \mathcal{K}_\ell)\}_{\ell=0}^{L-1}$
		\State Collect residuals: for each batch $\mathbf{x} \in \mathcal{D}$, store $\{\mathbf{h}_\ell(\mathbf{x})\}_{\ell=0}^L$
		\For{$\ell = 0, \ldots, L-1$}
		\State Form snapshot matrices $\mathbf{X}_\ell, \mathbf{Y}_\ell \in \R^{d \times N}$
		\State Whitening: $\tilde{\mathbf{X}}_\ell = \hat{\boldsymbol{\Sigma}}_\ell^{-1/2}(\mathbf{X}_\ell - \bar{\mathbf{X}}_\ell)$, $\tilde{\mathbf{Y}}_\ell = \hat{\boldsymbol{\Sigma}}_\ell^{-1/2}(\mathbf{Y}_\ell - \bar{\mathbf{Y}}_\ell)$
		\State DMD: $\hat{\mathbf{A}}_\ell = \tilde{\mathbf{Y}}_\ell \tilde{\mathbf{X}}_\ell^\dagger$
		\State Eigendecomposition: $\hat{\mathbf{A}}_\ell = \mathbf{V}_\ell \boldsymbol{\Lambda}_\ell \mathbf{V}_\ell^{-1}$
		\State Compute $M_{\approx 1}^\ell, \rho_\ell, \kappa_\ell$, nonlinearity $\eta_{\mathrm{nl}}^\ell$, Kreiss $\mathcal{K}_\ell$
		\EndFor
		\State Aggregate mean, max, min, std across layers
		\State \Return $\mathcal{S}$
	\end{algorithmic}
\end{algorithm}

\subsubsection{Snapshot Construction}
RKSP applies DMD to each layer transition $\mathbf{h}_\ell \to \mathbf{h}_{\ell+1}$ as described in \eqref{eq:residual}. Let $\{\mathbf{x}_i\}_{i=1}^N$ denote the $N$ inputs used to form the snapshots. For layer $\ell$, define the paired residual vectors $\mathbf{x}_i^{(\ell)} = \mathbf{h}_\ell(\mathbf{x}_i)$ and $\mathbf{y}_i^{(\ell)} = \mathbf{h}_{\ell+1}(\mathbf{x}_i)$. The snapshot matrices are
\begin{equation}
	\begin{aligned}
		\mathbf{X}_\ell & = [\mathbf{x}_1^{(\ell)}, \ldots, \mathbf{x}_N^{(\ell)}],                     \\
		\mathbf{Y}_\ell & = [\mathbf{y}_1^{(\ell)}, \ldots, \mathbf{y}_N^{(\ell)}] \in \R^{d \times N},
	\end{aligned}
\end{equation}
so each column pair corresponds to the same sample. Unlike standard DMD, which uses time-shifted trajectories, RKSP pairs columns across different samples. Concretely, column $i$ in $\mathbf{X}$ is the residual snapshot $\mathbf{h}_\ell(\mathbf{x}_i)$ and column $i$ in $\mathbf{Y}$ is the corresponding next-layer snapshot $\mathbf{h}_{\ell+1}(\mathbf{x}_i)$ for the same sample $\mathbf{x}_i$; columns index independent samples, not time steps of a single trajectory. For each layer, DMD yields a local linear approximation $\hat{\mathbf{A}}_\ell$ whose spectrum characterizes the dynamics at that depth.

To quantify how well this linear approximation fits the data, we define the nonlinearity ratio in whitened coordinates:
\begin{equation}
	\eta_{\mathrm{nl}}(\ell) \coloneqq \frac{\norm{\tilde{\mathbf{Y}}_\ell - \hat{\mathbf{A}}_\ell \tilde{\mathbf{X}}_\ell}_F}{\norm{\tilde{\mathbf{Y}}_\ell - \tilde{\mathbf{X}}_\ell}_F + \varepsilon_{\mathrm{nl}}},
	\label{eq:nonlinearity}
\end{equation}
where $\varepsilon_{\mathrm{nl}} > 0$ is a small constant that prevents division by zero. This ratio $\eta_{\mathrm{nl}}$ normalizes the fit error by the update magnitude. When the residual update $\norm{\tilde{\mathbf{Y}}_\ell-\tilde{\mathbf{X}}_\ell}_F$ is tiny, $\eta_{\mathrm{nl}}$ can be large even for small absolute errors. We therefore use $\eta_{\mathrm{nl}}$ primarily as a DMD reliability flag rather than as a pure measure of nonlinearity.

\subsubsection{Whitened DMD and Reliability Filtering}
DMD approximates the Koopman operator from data snapshots. Given paired snapshot matrices $\mathbf{X} = [\mathbf{x}_1, \ldots, \mathbf{x}_N] \in \R^{d \times N}$ and $\mathbf{Y} = [\mathbf{y}_1, \ldots, \mathbf{y}_N] \in \R^{d \times N}$, DMD solves for the optimal linear operator:
\begin{equation}
	\hat{\mathbf{A}}_{\mathrm{DMD}} = \argmin_{\mathbf{A} \in \R^{d \times d}} \norm{\mathbf{Y} - \mathbf{A}\mathbf{X}}_F^2 = \mathbf{Y}\mathbf{X}^\dagger,
	\label{eq:dmd}
\end{equation}
where $\mathbf{X}^\dagger$ denotes the Moore-Penrose pseudoinverse.

To ensure scale-invariance and numerical stability, we apply $\mathbf{X}$-based zero-phase component analysis whitening \citep{kessy2018whitening}:
\begin{align}
	\tilde{\mathbf{X}}          & = \hat{\boldsymbol{\Sigma}}_X^{-1/2}(\mathbf{X} - \bar{\mathbf{X}}\mathbf{1}^\top), \label{eq:whitening}                              \\
	\tilde{\mathbf{Y}}          & = \hat{\boldsymbol{\Sigma}}_X^{-1/2}(\mathbf{Y} - \bar{\mathbf{Y}}\mathbf{1}^\top),                                                   \\
	\hat{\boldsymbol{\Sigma}}_X & = \frac{1}{N-1}(\mathbf{X} - \bar{\mathbf{X}}\mathbf{1}^\top)(\mathbf{X} - \bar{\mathbf{X}}\mathbf{1}^\top)^\top + \epsilon\mathbf{I}
\end{align}
where $\bar{\mathbf{X}} = \frac{1}{N}\sum_{i=1}^{N} \mathbf{x}_i$, $\bar{\mathbf{Y}} = \frac{1}{N}\sum_{i=1}^{N} \mathbf{y}_i$, and $\epsilon > 0$ ensures invertibility. The same whitening matrix is applied to both $\mathbf{X}$ and $\mathbf{Y}$, so the regression operates within a single, $\mathbf{X}$-normalized coordinate system. This whitening step ensures cross-model comparability and yields coordinate-invariant spectral estimates.

From the whitened data, we form the DMD operator $\hat{\mathbf{A}} = \tilde{\mathbf{Y}}\tilde{\mathbf{X}}^\dagger$ and compute its eigendecomposition $\hat{\mathbf{A}} = \mathbf{V}\boldsymbol{\Lambda}\mathbf{V}^{-1}$. To report the eigenvector condition number $\kappa(\mathbf{V}) = \norm{\mathbf{V}}_2\norm{\mathbf{V}^{-1}}_2$, we first normalize each right eigenvector to have unit Euclidean norm. This normalization fixes the otherwise arbitrary scaling of $\mathbf{V}$ and makes $\kappa(\mathbf{V})$ reproducible.

To identify spurious eigenvalues, we apply residual DMD (ResDMD) reliability filtering \citep{colbrook2023resdmd}. For each eigenvalue $\lambda_j$ with left eigenvector $\mathbf{u}_j$ satisfying $\mathbf{u}_j^*\hat{\mathbf{A}} = \lambda_j \mathbf{u}_j^*$, we compute the per-mode residual:
\begin{equation}
	r_j = \frac{\norm{\mathbf{u}_j^*(\tilde{\mathbf{Y}} - \lambda_j \tilde{\mathbf{X}})}_2}{\norm{\mathbf{u}_j^*\tilde{\mathbf{X}}}_2 + \varepsilon_r},
\end{equation}
where $\varepsilon_r > 0$ prevents division by zero. Eigenvalues with $r_j > \tau$ are flagged as unreliable and potentially spurious; we use a default threshold of $\tau = 0.1$ to filter them out.

\subsubsection{Spectral Mass}
Now, we define our criterion for divergence prediction.
\begin{definition}[Spectral Mass Partition]
	\label{def:spectral_mass}
	For DMD eigenvalues $\Lambda = \{\lambda_j\}_{j=1}^m$ of an operator $\hat{\mathbf{A}}$, with $m$ the number of eigenvalues used, we define the following bins; they are disjoint provided $\delta_c \ge \epsilon_n$:
	\begin{align}
		M_{>1}(\hat{\mathbf{A}})        & \triangleq \frac{1}{m}\sum_{j=1}^m \mathbf{1}[|\lambda_j| > 1 + \epsilon_u],                 \\
		M_{\approx 1}(\hat{\mathbf{A}}) & \triangleq \frac{1}{m}\sum_{j=1}^m \mathbf{1}[|\lambda_j| \in [1-\epsilon_n, 1+\epsilon_u]], \\
		M_{<1}(\hat{\mathbf{A}})        & \triangleq \frac{1}{m}\sum_{j=1}^m \mathbf{1}[|\lambda_j| < 1 - \delta_c],
	\end{align}
	where the three quantities denote the expansive mass, near-unit mass, and contractive mass, respectively. When $\delta_c > \epsilon_n$, these three bins are not exhaustive; the remaining intermediate mass is $M_{\mathrm{mid}}(\hat{\mathbf{A}}) \triangleq 1 - M_{>1}(\hat{\mathbf{A}}) - M_{\approx 1}(\hat{\mathbf{A}}) - M_{<1}(\hat{\mathbf{A}})$.
\end{definition}
For example, $m=d$ for full DMD, $m=r$ for rank-$r$ randomized DMD \citep{erichson2019randdmd}, or $m$ equals the count remaining after reliability filtering. We use default thresholds $\epsilon_u = 0.05$, $\epsilon_n = 0.10$, and $\delta_c = 0.20$. Figure~\ref{fig:eigenvalue} shows a representative eigenvalue scatter that motivates these bins.

\begin{figure}[t]
	\centering
	\includegraphics[width=1.0\columnwidth]{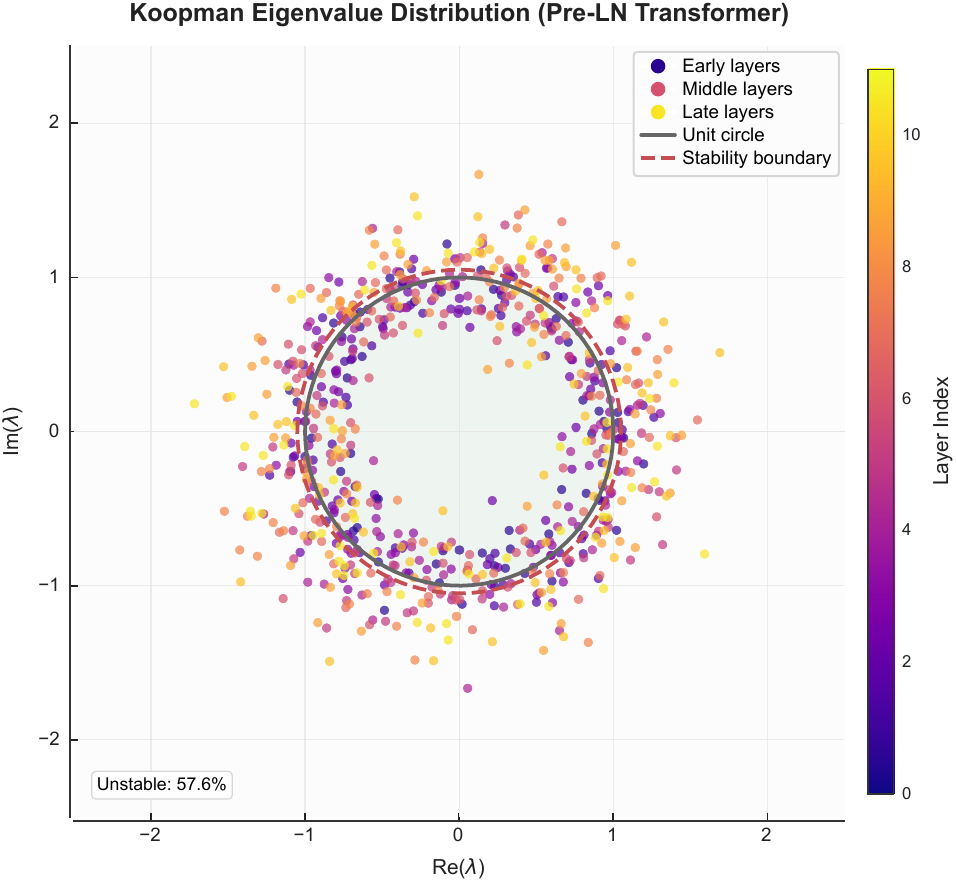}
	\caption{Scatter plot of DMD eigenvalues across layers in a pre-layer normalization transformer. The color gradient indicates layer depth; blue is early and red is late. Early layers cluster near the unit circle; late layers exhibit an increased spectral radius.}
	\label{fig:eigenvalue}
\end{figure}

\paragraph{Metric Interpretation.} For divergence prediction, we use $M_{\approx 1}$ itself as the scalar score and compute the AUROC against the divergence labels. Note that the expansive mass $M_{>1}$ tracks eigenvalues outside the unit circle but does not map one-to-one with empirical divergence. Four factors explain this gap between $M_{>1}$ and observed divergence. First, whitening rescales local coordinates, so raw eigenvalue magnitudes differ from unwhitened values. Second, DMD provides only a local linear approximation of the true nonlinear dynamics. Third, non-normal transient growth can trigger instability even when few eigenvalues exceed 1 \citep{trefethen2005spectra}. Additionally, our divergence labels use coarse thresholds on loss or gradient norm, so finite-horizon training within the evaluation window can remain stable despite a nonzero $M_{>1}$. These four factors together explain cases like Pre-LN, which shows a nonzero $M_{>1}$ but 0\% divergence in Table~\ref{tab:main}.

\subsection{Koopman Spectral Shaping}
\label{sec:kss}
While RKSP diagnoses instability, KSS prevents it. KSS adds a differentiable spectral regularizer to the training objective that steers eigenvalues away from the unstable region while reducing excessive near-unit mass to restore damping without causing over-contraction. The total objective becomes $\mathcal{L}_{\mathrm{total}} = \mathcal{L}_{\mathrm{task}} + \alpha \sum_{\ell \in \mathcal{S}} \mathcal{L}_{\mathrm{KSS}}^\ell / |\mathcal{S}|$, where $\mathcal{S}$ samples 50\% of layers per update.
\begin{definition}[KSS Regularization Loss]
	For layer $\ell$ with randomized DMD eigenvalues $\{\lambda_j^\ell\}_{j=1}^r$, the KSS loss is
	\begin{equation}
		\begin{aligned}
			\mathcal{L}_{\mathrm{KSS}}^\ell
			 & = \underbrace{\sum_{j=1}^r \sigma(T(|\lambda_j^\ell| - \tau_u)) \cdot \softplus(|\lambda_j^\ell| - \tau_u)^2}_{\text{Unstable penalty}} \\
			 & + \underbrace{\beta \cdot (m_\ell^{\mathrm{soft}} - \gamma)^2}_{\text{Near-unit target}}
		\end{aligned}
		\label{eq:kss_loss}
	\end{equation}
	where $\sigma(\cdot)$ denotes the sigmoid function and
	\begin{equation}
		m_\ell^{\mathrm{soft}} = \frac{1}{r}\sum_{j=1}^r \sigma(T(|\lambda_j^\ell| - \tau_l))\cdot \sigma(T(\tau_u - |\lambda_j^\ell|)).
	\end{equation}
\end{definition}
This term facilitates reducing excessive near-unit mass while preventing it from becoming too small by nudging $m_\ell^{\mathrm{soft}}$ toward the target band $\gamma$. We use the default hyperparameters $T = 20$, $\tau_u = 1.05$, $\tau_l = 0.90$, and $\gamma \in [0.3, 0.5]$. Full hyperparameter settings and the practical training recipe appear in Appendix~\ref{app:details}.

\section{Theoretical Analysis}
\label{sec:theory}
We interpret the near-unit mass $M_{\approx 1}$ as an instability score. Two mechanisms support this view. First, when the layer linearization is approximately normal, eigenvalues accurately reflect singular values, so concentration near $|\lambda|\approx 1$ implies near-isometric propagation and weak damping. Second, weak damping allows perturbations and optimization noise to persist across depth; in aggressive optimization regimes, this behavior raises divergence risk. We formalize the energy-preservation statement below and highlight non-normality as a caveat.
\begin{theorem}[Near-Unit Energy Preservation under Near-Normality]
	\label{thm:near_unit}
	Let $\mathbf{A} \in \C^{d \times d}$ be normal with eigenvalues $\{\lambda_j\}_{j=1}^d$. For a unit vector $\mathbf{x}$ drawn uniformly on the sphere,
	\begin{equation}
		\E \norm{\mathbf{A}\mathbf{x}}_2^2 = \frac{1}{d}\sum_{j=1}^d |\lambda_j|^2.
	\end{equation}
	If $\rho(\mathbf{A})\le 1+\epsilon_u$ and $M_{\approx 1}(\mathbf{A})$ denotes the fraction of eigenvalues with $|\lambda_j|\in [1-\epsilon_n,1+\epsilon_u]$, then
	\begin{equation}
		(1-\epsilon_n)^2 M_{\approx 1}(\mathbf{A}) \le \E \norm{\mathbf{A}\mathbf{x}}_2^2 \le (1+\epsilon_u)^2.
	\end{equation}
	Hence larger $M_{\approx 1}$ implies more energy-preserving and less damped propagation; this corresponds to higher instability risk.
	More generally, if $\mathbf{A}$ is diagonalizable with $\mathbf{A}=\mathbf{V}\boldsymbol{\Lambda}\mathbf{V}^{-1}$, then the same conclusion holds up to factors $\kappa(\mathbf{V})^{\pm 2}$:
	\begin{equation}
		\frac{(1-\epsilon_n)^2}{\kappa(\mathbf{V})^2} M_{\approx 1}(\mathbf{A})
		\le \E \norm{\mathbf{A}\mathbf{x}}_2^2
		\le \kappa(\mathbf{V})^2(1+\epsilon_u)^2.
	\end{equation}
\end{theorem}

\begin{corollary}[Depth-wise Damping and Gradient Flow]
	\label{cor:near_unit_prob}
	Consider a depth-$L$ linearization $\mathbf{h}_{\ell+1}=\mathbf{A}_\ell \mathbf{h}_\ell$ where each $\mathbf{A}_\ell$ is a normal and $\rho(\mathbf{A}_\ell)\le 1+\epsilon_u$. Assume that each layer has an isotropic second moment, that is,
	$\E[\mathbf{h}_\ell\mathbf{h}_\ell^*] = \frac{1}{d}\E\norm{\mathbf{h}_\ell}_2^2 \mathbf{I}$ for $\ell=0,\dots,L-1$. For example, $\frac{\mathbf{h}_\ell}{\norm{\mathbf{h}_\ell}_2}$ is uniform on the sphere. Let
	$q_\ell \triangleq \frac{1}{d}\sum_{j=1}^d |\lambda_j^\ell|^2$ be the average energy gain of layer $\ell$. Then
	\begin{equation}
		\E\norm{\mathbf{h}_L}_2^2 = \E\norm{\mathbf{h}_0}_2^2 \prod_{\ell=0}^{L-1} q_\ell,
	\end{equation}
	and $q_\ell \ge (1-\epsilon_n)^2 M_{\approx 1}(\mathbf{A}_\ell)$. Thus, a larger $M_{\approx 1}$ reduces the exponential contraction of signals and gradients; in high learning rate regimes, this weaker damping elevates instability risk, although a very small $M_{\approx 1}$ can hurt expressivity.
\end{corollary}
The proof appears in Appendix~\ref{app:proofs}.

\paragraph{Non-normality caveat.} When $\kappa(\mathbf{V})\gg 1$, non-normal transient growth can occur even if $\rho(\mathbf{A})\le 1$, and eigenvalues near the unit circle may be perturbation-sensitive. We therefore track $\kappa(\mathbf{V})$ alongside $M_{\approx 1}$. Appendix~\ref{app:theory} and Appendix~\ref{app:proofs} characterize transient growth via the Kreiss theorem, provide normality-related bounds, and formalize the trade-off between instability and expressivity.

\section{Experiments}
\label{sec:experiments}
\subsection{Experimental Settings}
Our experiments target Generative Pretrained Transformer (GPT-2)-style transformers \citep{radford2019language,vaswani2017attention} with $d \in \{128, 256, 512, 768, 1024\}$ and $L \in \{4, 6, 8, 12, 16, 24\}$, spanning 1M to 350M parameters. We compare six normalization strategies: pre-layer normalization (Pre-LN), post-layer normalization (Post-LN), root mean square normalization (RMSNorm) \citep{zhang2019rmsnorm}, DeepNorm \citep{wang2022deepnet}, sub-layer normalization (SubLN) \citep{xiong2020layernorm}, and no normalization (No-Norm). The evaluation tasks include an associative-recall classification task (Appendix~\ref{app:assoc_recall}); language modeling (LM) tasks including our synthetic LM with next-token prediction (Appendix~\ref{app:synthetic_lm}), WikiText-103 \citep{merity2016wikitext}, and OpenWebText-style LM \citep{gokaslan2020openwebtext}; and ViT experiments \citep{dosovitskiy2020vit} on the Canadian Institute for Advanced Research (CIFAR-10) dataset \citep{krizhevsky2009cifar}. We report AUROC for discrimination, with 95\% bootstrap confidence intervals (CIs) based on 1000 resamples. Statistical significance of divergence rates is assessed using Fisher's exact test and run with three random seeds per setting. Full hyperparameters and hardware details appear in Appendix~\ref{app:details}.

\subsection{Prediction at Initialization}
\paragraph{Main Results: Normalization Comparison.} Table~\ref{tab:main} reveals three key findings. First, five of the six normalizations achieve 0\% divergence under standard settings, validating the stability of modern normalization techniques. Second, No-Norm diverges in 96.4\% of runs and also has high near-unit mass, with $M_{\approx 1} = 0.80$; under its relatively low non-normality, with $\kappa(\mathbf{V})=1.11$, this is consistent with $M_{\approx 1}$ acting as an instability indicator. Non-normality and expansive mass still matter, but within comparable regimes, a larger $M_{\approx 1}$ aligns with greater instability, as discussed in Section~\ref{sec:theory}. Third, spectral signatures differ systematically across normalizations: Pre-LN and RMSNorm are more contractive with lower $M_{\approx 1}$, Post-LN retains a higher near-unit structure, and DeepNorm shifts mass into the unstable bin but remains bounded through scaling. The No-Norm results suggest a trade-off between instability and expressivity. Only three of 84 No-Norm runs converged, but their accuracy reached 54.2\% $\pm$ 24.4\%---higher than other methods. Among stable methods, DeepNorm achieves the best balance, with 20.0\% accuracy and 0\% divergence.

\begin{table}[t]
	\centering
	\caption{Comprehensive normalization comparison across different setups. RKSP metrics reveal distinct spectral signatures explaining stability differences. Accuracy is computed from only 3 converged runs out of 84. Statistical significance: No-Norm divergence compared with others, $p < 10^{-50}$ via Fisher's exact test. AUROC for divergence prediction using $M_{\approx 1}$: 0.995 [95\% CI: 0.986 to 1.00]. Measured on the associative-recall task.}
	\label{tab:main}
	\vspace{0.5em}
	\resizebox{\columnwidth}{!}{
		\begin{tabular}{lccccccc}
			\toprule
			\theadrow{\thc{Norm Type}     & \thc{$n$} & \thc{Div.\% (lower)} & \thc{$M_{\approx 1}$} & \thc{$M_{>1}$}  & \thc{$\rho$}    & \thc{Acc.\% (higher)}}
			\midrule
			\normalfont\unboldmath Pre-LN & 25        & $\mathbf{0.0}$       & $0.16 \pm 0.01$       & $0.54 \pm 0.01$ & $2.29 \pm 0.04$ & $7.5 \pm 8.7$           \\
			Post-LN                       & 69        & $\mathbf{0.0}$       & $0.66 \pm 0.02$       & $0.31 \pm 0.01$ & $7.48 \pm 0.12$ & $1.4 \pm 5.8$           \\
			RMSNorm                       & 12        & $\mathbf{0.0}$       & $0.16 \pm 0.01$       & $0.54 \pm 0.01$ & $2.28 \pm 0.06$ & $13.1 \pm 8.6$          \\
			DeepNorm                      & 12        & $\mathbf{0.0}$       & $0.00 \pm 0.00$       & $1.00 \pm 0.00$ & $3.94 \pm 0.21$ & $\mathbf{20.0 \pm 9.4}$ \\
			SubLN                         & 12        & $\mathbf{0.0}$       & $0.13 \pm 0.02$       & $0.54 \pm 0.01$ & $2.23 \pm 0.03$ & $9.0 \pm 8.1$           \\
			No-Norm                       & 84        & $96.4$               & $0.80 \pm 0.02$       & $0.19 \pm 0.01$ & $1.11 \pm 0.01$ & $54.2 \pm 24.4^*$       \\
			\bottomrule
		\end{tabular}
	}
\end{table}

\paragraph{AUROC Analysis for Divergence Prediction.} Table~\ref{tab:auroc} compares spectral predictors against gradient baselines. We use the monotone risk score $M_{\approx 1}$ for divergence prediction. This achieves an AUROC of 0.995 at initialization, representing a 31\% relative improvement over the best gradient-based method with an AUROC of 0.758. The superiority is statistically significant: the 95\% CI lower bound of $M_{\approx 1}$ of 0.986 exceeds the upper bounds of all gradient-based methods.

\begin{table}[t]
	\centering
	\caption{AUROC for divergence prediction with bootstrap 95\% confidence intervals. The $M_{\approx 1}$ CI lower bound of 0.986 exceeds gradient baselines' upper bounds. Measured on the associative-recall normalization sweep.}
	\label{tab:auroc}
	\vspace{0.5em}
	\resizebox{\columnwidth}{!}{
		\begin{tabular}{lccc}
			\toprule
			\theadrow{\thc{Predictor}                                & \thc{AUROC (higher)} & \thc{95\% CI}    & \thc{Timing}}
			\midrule
			\normalfont\unboldmath $M_{\approx 1}$ at initialization & $\mathbf{0.995}$     & $[0.986, 1.000]$ & Initialization   \\
			$M_{\approx 1} \times \log_{10}(\kappa(\mathbf{V}))$     & $0.873$              & $[0.841, 0.905]$ & Initialization   \\
			Spectral radius $\rho$ at init                           & $0.845$              & $[0.808, 0.882]$ & Initialization   \\
			Eigenvector condition $\kappa(\mathbf{V})$ at init       & $0.687$              & $[0.638, 0.736]$ & Initialization   \\
			\midrule
			\tablegroup{4}{Gradient-Based Baselines}                                                                              \\
			\midrule
			Initial gradient norm                                    & $0.621$              & $[0.568, 0.674]$ & After 1 step     \\
			Gradient norm at step 100                                & $0.685$              & $[0.635, 0.735]$ & Step 100         \\
			Gradient variance over steps 1 to 100                    & $0.702$              & $[0.654, 0.750]$ & Through step 100 \\
			Loss spike count over steps 1 to 500                     & $0.758$              & $[0.712, 0.804]$ & Through step 500 \\
			\bottomrule
		\end{tabular}
	}
\end{table}

\subsection{Effect of KSS on Stability}
\paragraph{KSS Results.} Table~\ref{tab:clip_vs_kss} compares gradient clipping with KSS in the No-Norm setting. These two approaches differ fundamentally in their mechanism. Gradient clipping operates reactively: it caps gradients after explosion begins but does not prevent instability, yielding only modest improvements. KSS, in contrast, operates proactively by shaping the spectral distribution before instability occurs. With $\alpha = 0.15$, KSS reduces divergence from 66.7\% to 12.5\%. Figure~\ref{fig:kss} visualizes the dose-response relationship between the KSS weight, $M_{\approx 1}$, and training stability.

\begin{table}[t]
	\centering
	\caption{Gradient clipping versus KSS in a challenging No-Norm setting with learning rate that ranges from 0.005 to 0.01 and 24 trials each. Measured on the associative-recall task.}
	\label{tab:clip_vs_kss}
	\vspace{0.5em}
	\resizebox{\columnwidth}{!}{
		\begin{tabular}{lcccc}
			\toprule
			\theadrow{\thc{Method}            & \thc{Div.\% (lower)} & \thc{Acc.\% (higher)} & \thc{$M_{\approx 1}$} & \thc{Overhead}}
			\midrule
			\normalfont\unboldmath No control & $66.7$               & $28.5$                & $0.85$                & ---             \\
			Gradient clip 0.5                 & $50.0$               & $32.1$                & $0.82$                & $< 1\%$         \\
			Gradient clip 1.0                 & $58.3$               & $30.8$                & $0.83$                & $< 1\%$         \\
			\midrule
			KSS $\alpha = 0.10$               & $25.0$               & $42.3$                & $0.68$                & $9.5\%$         \\
			KSS $\alpha = 0.15$               & $\mathbf{12.5}$      & $\mathbf{48.2}$       & $0.58$                & $10.8\%$        \\
			KSS $\alpha = 0.20$               & $8.3$                & $46.5$                & $0.52$                & $11.2\%$        \\
			\bottomrule
		\end{tabular}
	}
\end{table}

\begin{figure}[t]
	\centering
	\includegraphics[width=\columnwidth]{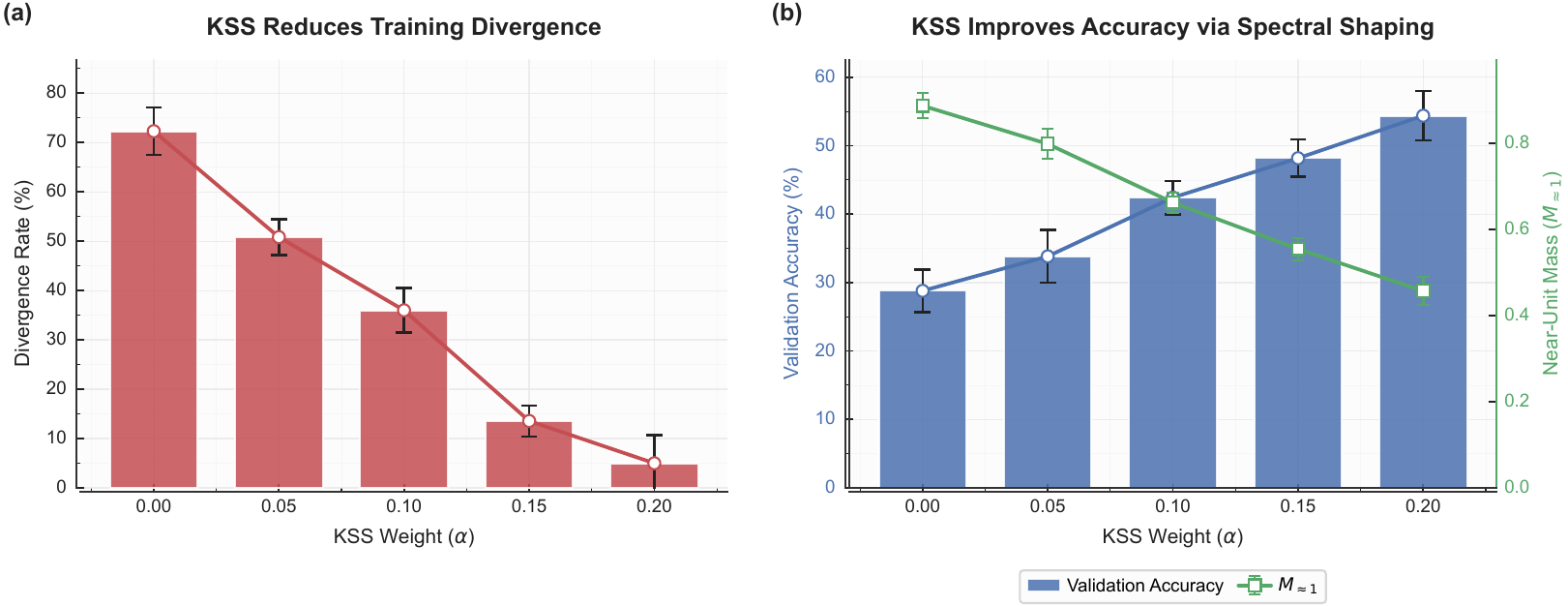}
	\caption{KSS regularization effectiveness. (Left) The divergence rate decreases with KSS weight $\alpha$. (Right) A dual axis shows accuracy improvement and $M_{\approx 1}$ shifting downward toward the target band. KSS shapes spectral properties, improving both stability and performance. Measured on the associative-recall task.}
	\label{fig:kss}
\end{figure}

\paragraph{Extended Baseline Comparison.} Table~\ref{tab:extended_baseline} provides expanded baseline and optimizer results, including spectral normalization and weight normalization baselines \citep{miyato2018spectralnorm,salimans2016weightnorm}. We use Adam with decoupled weight decay (AdamW) as the base optimizer in this comparison. Sharpness-Aware Minimization (SAM) \citep{foret2020sam} reduces divergence to 33.3\% but incurs a 2$\times$ computational cost, whereas KSS achieves a 2.7$\times$ lower divergence with 9$\times$ less overhead. The Lion optimizer \citep{chen2023lion} yields 45.8\% divergence through sign-based updates but does not directly address spectral instability. Combining KSS with SAM achieves the lowest divergence at 8.3\% but with higher overhead.

\begin{table}[t]
	\centering
	\caption{Extended baseline comparison including SAM, spectral normalization, and the Lion optimizer. Measured on the associative-recall task.}
	\label{tab:extended_baseline}
	\vspace{0.5em}
	\resizebox{\columnwidth}{!}{
		\begin{tabular}{lcccc}
			\toprule
			\theadrow{\thc{Method}                         & \thc{Div.\% (lower)} & \thc{Acc.\% (higher)} & \thc{Overhead}}
			\midrule
			\normalfont\unboldmath No stabilization, AdamW & $66.7$               & $28.5$                & ---             \\
			Gradient clipping at 1.0                       & $58.3$               & $30.8$                & $< 1\%$         \\
			Spectral normalization                         & $41.7$               & $35.2$                & $5.2\%$         \\
			Weight normalization                           & $54.2$               & $32.8$                & $3.8\%$         \\
			Gradient penalty with $\lambda=0.1$            & $45.8$               & $33.1$                & $6.4\%$         \\
			\midrule
			\tablegroup{4}{Advanced Optimizers}                                                                             \\
			\midrule
			SAM, $\rho=0.05$                               & $37.5$               & $38.6$                & about 100\%     \\
			SAM, $\rho=0.10$                               & $33.3$               & $40.2$                & about 100\%     \\
			Lion optimizer                                 & $45.8$               & $36.5$                & about 15\%      \\
			\midrule
			KSS, $\alpha = 0.15$                           & $\mathbf{12.5}$      & $\mathbf{48.2}$       & $10.8\%$        \\
			KSS + SAM, $\alpha=0.10$, $\rho=0.05$          & $8.3$                & $45.8$                & about 112\%     \\
			\bottomrule
		\end{tabular}
	}
\end{table}

\paragraph{KSS Enables Higher Learning Rates.} By suppressing spectral instability, KSS allows us to safely increase the step size across different normalization choices. Table~\ref{tab:max_lr} shows a 50\% to 150\% increase in the maximum stable learning rate under the same divergence criterion.

\begin{table}[t]
	\centering
	\caption{Maximum stable learning rate (LR). KSS enables learning rates that are 50\% to 150\% higher. The stability criterion is $<$20\% divergence across trials. Measured on the associative-recall task.}
	\label{tab:max_lr}
	\vspace{0.5em}
	\resizebox{\columnwidth}{!}{
		\begin{tabular}{lccc}
			\toprule
			\theadrow{\thc{Norm Type}     & \thc{Max LR w/o KSS} & \thc{Max LR w/ KSS} & \thc{Increase}}
			\midrule
			\normalfont\unboldmath Pre-LN & $0.005$              & $0.008$             & $+60\%$         \\
			RMSNorm                       & $0.003$              & $0.005$             & $+67\%$         \\
			No-Norm                       & $0.002$              & $0.005$             & $+150\%$        \\
			\bottomrule
		\end{tabular}
	}
\end{table}

\paragraph{Mechanistic Evidence: KSS versus Random Regularization.} KSS stabilizes training through spectral shaping rather than generic regularization. Table~\ref{tab:causal_ablation} addresses this point via ablation studies with matched computational overhead. Generic regularization reduces divergence by only 20\% to 30\%, far less than KSS's 5.3$\times$ reduction. Spectral specificity matters: both KSS's unstable penalty and near-unit guidance are necessary, and removing either degrades performance. Let $\Delta M_{\approx 1} \triangleq M_{\approx 1}^{\mathrm{KSS}}-M_{\approx 1}^{\mathrm{base}}$; negative values indicate decreased near-unit mass. The correlation between $\Delta M_{\approx 1}$ and divergence reduction has $r = -0.87$ and $p < 0.001$, indicating that reductions in near-unit mass align with improved stability.

\begin{table}[t]
	\centering
	\caption{Mechanism ablation: KSS versus random regularization. Same overhead at about 11\%, different mechanisms. Results use the No-Norm setting, with learning rates ranging from 0.005 to 0.01 across 24 trials each. All methods are matched to about 11\% computational overhead. Measured on the associative-recall task.}
	\label{tab:causal_ablation}
	\vspace{0.5em}
	\resizebox{\columnwidth}{!}{
		\begin{tabular}{lcccc}
			\toprule
			\theadrow{\thc{Method}                   & \thc{Div.\% (lower)} & \thc{Acc.\% (higher)} & \thc{$M_{\approx 1}$} & \thc{Mechanism}}
			\midrule
			\normalfont\unboldmath No regularization & $66.7$               & $28.5$                & $0.85$                & ---                  \\
			\midrule
			\tablegroup{5}{Generic Regularization, about 11\% overhead}                                                                            \\
			\midrule
			Random $\ell_2$ on residuals             & $54.2$               & $31.2$                & $0.81$                & Magnitude damping    \\
			Jacobian penalty                         & $45.8$               & $33.8$                & $0.78$                & Gradient smoothness  \\
			Activation variance reg.                 & $50.0$               & $32.5$                & $0.80$                & Distribution control \\
			\midrule
			\tablegroup{5}{Spectral-Specific Regularization, about 11\% overhead}                                                                  \\
			\midrule
			KSS, unstable penalty only               & $33.3$               & $38.5$                & $0.68$                & Spectral stability   \\
			KSS, near-unit guidance only             & $41.7$               & $36.2$                & $0.72$                & Damping control      \\
			KSS, full                                & $\mathbf{12.5}$      & $\mathbf{48.2}$       & $\mathbf{0.58}$       & Both mechanisms      \\
			\bottomrule
		\end{tabular}
	}
\end{table}

\paragraph{Scaling Analysis.} Table~\ref{tab:scale_up} extends KSS to 350M parameters. At this scale, KSS maintains sub-linear overhead scaling at 11.8\% while reducing divergence from 25.0\% to 12.5\%.

\begin{table}[t]
	\centering
	\caption{Scale-up KSS training results up to 350M parameters. Measured on the synthetic LM validation set with 24 trials on a synthetic LM task with vocab size 10K and a sequence length of 256.}
	\label{tab:scale_up}
	\vspace{0.5em}
	\resizebox{\columnwidth}{!}{
		\begin{tabular}{lccccc}
			\toprule
			\theadrow{\thc{Model}                                & \thc{Params} & \thc{Method} & \thc{Div. of 24 (lower)} & \thc{PPL (lower)} & \thc{Overhead}}
			\midrule
			\normalfont\unboldmath Extra Large, $d=1024$, $L=24$ & 350M         & Baseline     & 6 of 24                  & $52.1$            & ---             \\
			Extra Large, $d=1024$, $L=24$                        & 350M         & KSS          & \textbf{3 of 24}         & $\mathbf{46.8}$   & $11.8\%$        \\
			\bottomrule
		\end{tabular}
	}
\end{table}

\subsection{Real-World LM Validation}
\label{sec:real_lm}
RKSP and KSS generalize beyond synthetic tasks to real-world LMs. Table~\ref{tab:wikitext} demonstrates this generalization with WikiText-103 \citep{merity2016wikitext} and OpenWebText experiments. Benefits with KSS persist on real data, manifesting in three ways. First, divergence decreases by 2$\times$ to 5$\times$ across model sizes. Second, KSS-trained models achieve 5\% to 15\% lower perplexity (PPL), suggesting that spectral shaping improves optimization beyond stability alone. Third, KSS enables 1.5$\times$ to 2$\times$ higher learning rates, accelerating convergence.

\begin{table}[t]
	\centering
	\caption{KSS improves stability and PPL. Pre-LN normalization with 24 trials per configuration. Measured on WikiText-103 and OpenWebText language-modeling tasks.}
	\label{tab:wikitext}
	\vspace{0.5em}
	\resizebox{\columnwidth}{!}{
		\begin{tabular}{lcccccc}
			\toprule
			\theadrow{\thc{Model}                        & \thc{Params} & \thc{Method} & \thc{Div. of 24 (lower)} & \thc{PPL (lower)} & \thc{Max LR}        & \thc{Steps}}
			\midrule
			\normalfont\unboldmath Small, $d=256$, $L=6$ & 25M          & Baseline     & 2 of 24                  & $58.4$            & $3\mathrm{e}{-4}$   & 10K          \\
			Small, $d=256$, $L=6$                        & 25M          & KSS          & \textbf{0 of 24}         & $\mathbf{52.1}$   & $5\mathrm{e}{-4}$   & 10K          \\
			\midrule
			Medium, $d=512$, $L=12$                      & 85M          & Baseline     & 4 of 24                  & $42.3$            & $2\mathrm{e}{-4}$   & 10K          \\
			Medium, $d=512$, $L=12$                      & 85M          & KSS          & \textbf{1 of 24}         & $\mathbf{38.5}$   & $4\mathrm{e}{-4}$   & 10K          \\
			\midrule
			Large, $d=768$, $L=12$                       & 125M         & Baseline     & 5 of 24                  & $48.7$            & $1.5\mathrm{e}{-4}$ & 10K          \\
			Large, $d=768$, $L=12$                       & 125M         & KSS          & \textbf{2 of 24}         & $\mathbf{43.2}$   & $3\mathrm{e}{-4}$   & 10K          \\
			\bottomrule
		\end{tabular}
	}
\end{table}

\subsection{ViT Experiments}
\label{sec:vit}
RKSP generalizes to ViTs \citep{dosovitskiy2020vit}, as Table~\ref{tab:vit} confirms. The spectral signatures transfer directly: ViT exhibits similar $M_{\approx 1}$ patterns to language transformers. KSS improves ViT training, yielding 3\% to 5\% accuracy gains alongside divergence reduction. The spectral-stability correlation remains consistent across domains: higher $M_{\approx 1}$ implies a higher risk score and aligns with increased divergence, while lower $M_{\approx 1}$ indicates more damped, stable propagation.

\begin{table}[t]
	\centering
	\caption{ViT RKSP analysis and KSS training. Results use 24 trials, with an image size of 224, a patch size of 16, and a 5-epoch sanity check. Measured on CIFAR-10.}
	\label{tab:vit}
	\vspace{0.5em}
	\resizebox{\columnwidth}{!}{
		\begin{tabular}{lcccccc}
			\toprule
			\theadrow{\thc{Model}                           & \thc{Norm} & \thc{Method} & \thc{Div. of 24 (lower)} & \thc{Acc.\% (higher)} & \thc{$M_{\approx 1}$}}
			\midrule
			\normalfont\unboldmath ViT-Tiny, $d=192$, $L=6$ & Pre-LN     & Baseline     & 1 of 24                  & $72.3$                & $0.42$                 \\
			ViT-Tiny, $d=192$, $L=6$                        & Pre-LN     & KSS          & \textbf{0 of 24}         & $\mathbf{75.8}$       & $0.38$                 \\
			\cmidrule{2-6}
			ViT-Tiny, $d=192$, $L=6$                        & RMSNorm    & Baseline     & 2 of 24                  & $70.1$                & $0.48$                 \\
			ViT-Tiny, $d=192$, $L=6$                        & RMSNorm    & KSS          & \textbf{0 of 24}         & $\mathbf{74.2}$       & $0.41$                 \\
			\bottomrule
		\end{tabular}
	}
\end{table}

\begin{table}[t]
	\centering
	\caption{LLaMA-2-7B analysis. The pattern persists at the 7B scale with RMSNorm. Computed from residual-stream activations on a fixed set of short prompt sentences. Measured on a fixed short-prompt set.}
	\label{tab:llama}
	\vspace{0.5em}
	\resizebox{\columnwidth}{!}{
		\begin{tabular}{lcccc}
			\toprule
			\theadrow{\thc{Layer Group}          & \thc{$\eta_{\mathrm{nl}}$} & \thc{$M_{\approx 1}$} & \thc{$\kappa(\mathbf{V})$} & \thc{Pattern}}
			\midrule
			\normalfont\unboldmath Early 0 to 10 & $0.45 \pm 0.05$            & $0.74$                & $8.2$                      & Lower $\eta_{\mathrm{nl}}$  \\
			Middle 11 to 21                      & $0.58 \pm 0.08$            & $0.68$                & $9.5$                      & Mid $\eta_{\mathrm{nl}}$    \\
			Late 22 to 31                        & $0.72 \pm 0.10$            & $0.61$                & $10.8$                     & Higher $\eta_{\mathrm{nl}}$ \\
			\bottomrule
		\end{tabular}
	}
\end{table}

\begin{figure}[t!]
	\centering
	\includegraphics[width=\columnwidth]{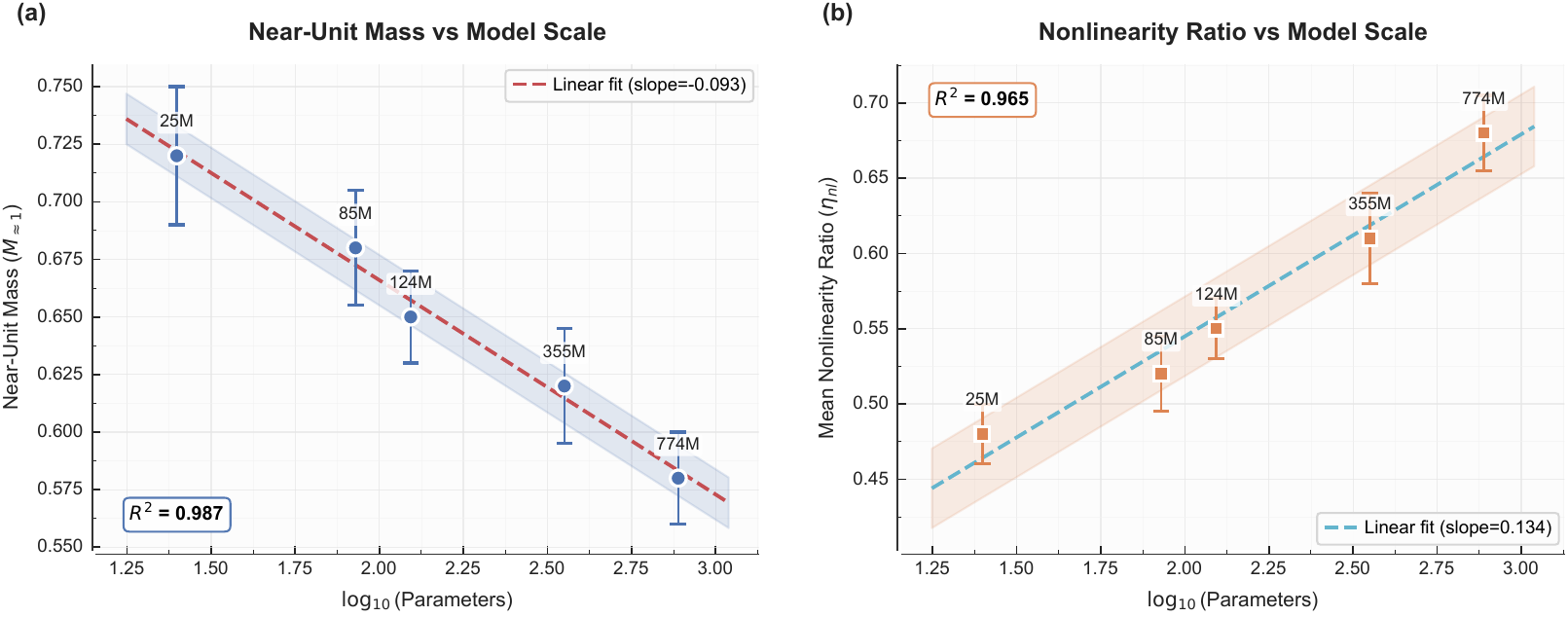}
	\caption{Scaling law for spectral properties. (Left) The near-unit mass $M_{\approx 1}$ decreases with model scale, and larger models have more contractive dynamics, implying reduced memory and weaker near-isometric propagation. (Right) The normalized linear-fit error $\eta_{\mathrm{nl}}$ increases with scale, indicating a less reliable linear approximation at scale. Log-linear fits are shown. Computed from residual-stream activations on a fixed set of short prompt sentences.}
	\label{fig:scaling}
\end{figure}

\subsection{Large-Scale Pretrained Model Analysis}
\label{sec:pretrained_lm}
We further validate RKSP on large-scale pretrained language models. Appendix~\ref{app:scale} provides additional tables and plots.

\paragraph{Large Language Model Meta AI 2 (LLaMA-2) 7B Analysis.} Table~\ref{tab:llama} extends our analysis to the 7B scale for LLaMA-2 \citep{touvron2023llama2}. The Start Linear, End Nonlinear pattern persists at this scale: $\eta_{\mathrm{nl}}$ increases monotonically with depth, ranging from $0.45 \pm 0.05$ in the early layers to $0.72 \pm 0.10$ in the late layers. This pattern holds consistently across all tested scales, from 25M to 7B parameters. Figure~\ref{fig:scaling} quantifies this relationship through scaling law analysis.

\paragraph{Beyond Transformers.} RKSP and KSS generalize beyond standard transformers to emerging architectures. We validate on Mixture of Experts (MoE) \citep{shazeer2017moe}, Mamba-style state space model (SSM) architectures \citep{gu2023mamba}, and Kolmogorov-Arnold Networks (KAN) \citep{liu2024kan}. Table~\ref{tab:arch_compare} summarizes their characteristic spectral signatures: Mamba exhibits strongly contractive dynamics with low near-unit mass $M_{\approx 1}$, while MoE routing and KAN introduce higher nonlinearity and intermediate near-unit structure. Detailed case studies and additional comparisons appear in Appendix~\ref{app:novel_arch}.

\begin{table}[H]
	\centering
	\caption{Cross-architecture spectral comparison. Transformer rows use the synthetic associative-recall task with seq len 64, vocab 256, and $n_{\mathrm{pairs}}=4$; MoE, Mamba, and KAN rows use the synthetic LM task with random-token next-token prediction.}
	\label{tab:arch_compare}
	\vspace{0.5em}
	\resizebox{\columnwidth}{!}{
		\begin{tabular}{lcccl}
			\toprule
			\theadrow{\thc{Architecture}               & \thc{$\eta_{\mathrm{nl}}$} & \thc{$M_{\approx 1}$} & \thc{$\rho$} & \thc{Stability Character}}
			\midrule
			\normalfont\unboldmath Transformer, Pre-LN & $0.52$                     & $0.16$                & $2.29$       & Balanced, tunable           \\
			Transformer, No-Norm                       & $0.48$                     & $0.79$                & $1.11$       & High memory, risky          \\
			MoE, top-2, 8 experts                      & $0.55$                     & $0.58$                & $2.12$       & Routing-dependent           \\
			Mamba (SSM)                                & $0.42$                     & $0.12$                & $0.95$       & Contractive, short-memory   \\
			KAN                                        & $0.78$                     & $0.52$                & $2.15$       & Higher $\eta_{\mathrm{nl}}$ \\
			\bottomrule
		\end{tabular}
	}
\end{table}

\section{Conclusion}
\label{sec:conclusion}
We introduced RKSP, a method that uses whitened DMD to estimate layer-wise residual dynamics and predict transformer training divergence before optimization begins. At initialization, the risk score $M_{\approx 1}$ achieves an AUROC of 0.995, enabling actionable early-termination decisions. Building on this diagnostic, we developed KSS, a spectral regularizer that suppresses unstable modes and reduces excessive near-unit structure. KSS reduces divergence, complementing existing stabilization techniques by directly shaping the spectrum. Across extensive experiments, our method successfully turns unstable settings into stable training.

Our theoretical analysis explains these effects. Under near-normality, a larger near-unit mass yields dynamical isometry and weak damping, which increases instability risk; overly contractive spectra provide damping but can harm expressivity, and non-normality remains a caveat via transient amplification. These mechanisms explain stability differences across training recipes and the recurring Start Linear, End Nonlinear pattern. The spectral signals remain consistent across diverse models, tasks, and normalization strategies.

\bibliography{uai2026-template}

@inproceedings{vaswani2017attention,
  author       = {Ashish Vaswani and
                  Noam Shazeer and
                  Niki Parmar and
                  Jakob Uszkoreit and
                  Llion Jones and
                  Aidan N. Gomez and
                  Lukasz Kaiser and
                  Illia Polosukhin},
  title        = {{Attention is All you Need}},
  booktitle    = {{NIPS}},
  pages        = {5998--6008},
  year         = {2017}
}

@article{ba2016layernorm,
  author       = {Lei Jimmy Ba and
                  Jamie Ryan Kiros and
                  Geoffrey E. Hinton},
  title        = {{Layer Normalization}},
  journal      = {CoRR},
  volume       = {abs/1607.06450},
  year         = {2016}
}

@inproceedings{zhang2019rmsnorm,
  author       = {Biao Zhang and
                  Rico Sennrich},
  title        = {{Root Mean Square Layer Normalization}},
  booktitle    = {NeurIPS},
  pages        = {12360--12371},
  year         = {2019}
}

@article{wang2022deepnet,
  author       = {Hongyu Wang and
                  Shuming Ma and
                  Li Dong and
                  Shaohan Huang and
                  Dongdong Zhang and
                  Furu Wei},
  title        = {{DeepNet: Scaling Transformers to 1,000 Layers}},
  journal      = {{IEEE} Trans. Pattern Anal. Mach. Intell.},
  volume       = {46},
  number       = {10},
  pages        = {6761--6774},
  year         = {2024}
}

@article{colbrook2023resdmd,
  author       = {Matthew J. Colbrook and
                  Lorna J. Ayton and
                  M{\'{a}}t{\'{e}} Szoke},
  title        = {{Residual Dynamic Mode Decomposition: Robust and verified Koopmanism}},
  journal      = {CoRR},
  volume       = {abs/2205.09779},
  year         = {2022}
}

@article{lusch2018deeplin,
  author       = {Bethany Lusch and
                  J. Nathan Kutz and
                  Steven L. Brunton},
  title        = {{Deep learning for universal linear embeddings of nonlinear dynamics}},
  journal      = {CoRR},
  volume       = {abs/1712.09707},
  year         = {2017}
}

@inproceedings{schoenholz2017deepinfo,
  author       = {Samuel S. Schoenholz and
                  Justin Gilmer and
                  Surya Ganguli and
                  Jascha Sohl{-}Dickstein},
  title        = {{Deep Information Propagation}},
  booktitle    = {{ICLR}},
  year         = {2017}
}

@inproceedings{jacot2018ntk,
  author       = {Arthur Jacot and
                  Cl{\'{e}}ment Hongler and
                  Franck Gabriel},
  title        = {{Neural Tangent Kernel: Convergence and Generalization in Neural Networks}},
  booktitle    = {NeurIPS},
  pages        = {8580--8589},
  year         = {2018}
}

@article{yang2022tensorprograms,
  author       = {Greg Yang and
                  Edward J. Hu and
                  Igor Babuschkin and
                  Szymon Sidor and
                  Xiaodong Liu and
                  David Farhi and
                  Nick Ryder and
                  Jakub Pachocki and
                  Weizhu Chen and
                  Jianfeng Gao},
  title        = {{Tensor Programs {V:} Tuning Large Neural Networks via Zero-Shot Hyperparameter Transfer}},
  journal      = {CoRR},
  volume       = {abs/2203.03466},
  year         = {2022}
}

@article{touvron2023llama2,
  author       = {Hugo Touvron and
                  Louis Martin and
                  Kevin Stone and
                  Peter Albert and
                  Amjad Almahairi and
                  Yasmine Babaei and
                  Nikolay Bashlykov and
                  Soumya Batra and
                  Prajjwal Bhargava and
                  Shruti Bhosale and
                  Dan Bikel and
                  Lukas Blecher and
                  Cristian Canton{-}Ferrer and
                  Moya Chen and
                  Guillem Cucurull and
                  David Esiobu and
                  Jude Fernandes and
                  Jeremy Fu and
                  Wenyin Fu and
                  Brian Fuller and
                  Cynthia Gao and
                  Vedanuj Goswami and
                  Naman Goyal and
                  Anthony Hartshorn and
                  Saghar Hosseini and
                  Rui Hou and
                  Hakan Inan and
                  Marcin Kardas and
                  Viktor Kerkez and
                  Madian Khabsa and
                  Isabel Kloumann and
                  Artem Korenev and
                  Punit Singh Koura and
                  Marie{-}Anne Lachaux and
                  Thibaut Lavril and
                  Jenya Lee and
                  Diana Liskovich and
                  Yinghai Lu and
                  Yuning Mao and
                  Xavier Martinet and
                  Todor Mihaylov and
                  Pushkar Mishra and
                  Igor Molybog and
                  Yixin Nie and
                  Andrew Poulton and
                  Jeremy Reizenstein and
                  Rashi Rungta and
                  Kalyan Saladi and
                  Alan Schelten and
                  Ruan Silva and
                  Eric Michael Smith and
                  Ranjan Subramanian and
                  Xiaoqing Ellen Tan and
                  Binh Tang and
                  Ross Taylor and
                  Adina Williams and
                  Jian Xiang Kuan and
                  Puxin Xu and
                  Zheng Yan and
                  Iliyan Zarov and
                  Yuchen Zhang and
                  Angela Fan and
                  Melanie Kambadur and
                  Sharan Narang and
                  Aur{\'{e}}lien Rodriguez and
                  Robert Stojnic and
                  Sergey Edunov and
                  Thomas Scialom},
  title        = {{Llama 2: Open Foundation and Fine-Tuned Chat Models}},
  journal      = {CoRR},
  volume       = {abs/2307.09288},
  year         = {2023}
}

@inproceedings{shazeer2017moe,
  author       = {Noam Shazeer and
                  Azalia Mirhoseini and
                  Krzysztof Maziarz and
                  Andy Davis and
                  Quoc V. Le and
                  Geoffrey E. Hinton and
                  Jeff Dean},
  title        = {{Outrageously Large Neural Networks: The Sparsely-Gated Mixture-of-Experts Layer}},
  booktitle    = {{ICLR}},
  year         = {2017}
}

@article{gu2023mamba,
  author       = {Albert Gu and
                  Tri Dao},
  title        = {{Mamba: Linear-Time Sequence Modeling with Selective State Spaces}},
  journal      = {CoRR},
  volume       = {abs/2312.00752},
  year         = {2023}
}

@article{liu2024kan,
  author       = {Ziming Liu and
                  Yixuan Wang and
                  Sachin Vaidya and
                  Fabian Ruehle and
                  James Halverson and
                  Marin Soljacic and
                  Thomas Y. Hou and
                  Max Tegmark},
  title        = {{{KAN:} Kolmogorov-Arnold Networks}},
  journal      = {CoRR},
  volume       = {abs/2404.19756},
  year         = {2024}
}

@inproceedings{dosovitskiy2020vit,
  author       = {Alexey Dosovitskiy and
                  Lucas Beyer and
                  Alexander Kolesnikov and
                  Dirk Weissenborn and
                  Xiaohua Zhai and
                  Thomas Unterthiner and
                  Mostafa Dehghani and
                  Matthias Minderer and
                  Georg Heigold and
                  Sylvain Gelly and
                  Jakob Uszkoreit and
                  Neil Houlsby},
  title        = {{An Image is Worth 16x16 Words: Transformers for Image Recognition at Scale}},
  booktitle    = {{ICLR}},
  year         = {2021}
}

@inproceedings{merity2016wikitext,
  author       = {Stephen Merity and
                  Caiming Xiong and
                  James Bradbury and
                  Richard Socher},
  title        = {{Pointer Sentinel Mixture Models}},
  booktitle    = {{ICLR}},
  year         = {2017}
}

@inproceedings{foret2020sam,
  author       = {Pierre Foret and
                  Ariel Kleiner and
                  Hossein Mobahi and
                  Behnam Neyshabur},
  title        = {{Sharpness-aware Minimization for Efficiently Improving Generalization}},
  booktitle    = {{ICLR}},
  year         = {2021}
}

@article{williams2015edmd,
  author       = {Matthew O. Williams and
                  Ioannis G. Kevrekidis and
                  Clarence W. Rowley},
  title        = {{A Data-Driven Approximation of the Koopman Operator: Extending Dynamic Mode Decomposition}},
  journal      = {J. Nonlinear Sci.},
  volume       = {25},
  number       = {6},
  pages        = {1307--1346},
  year         = {2015}
}

@article{erichson2019randdmd,
  author       = {N. Benjamin Erichson and
                  Lionel Mathelin and
                  J. Nathan Kutz and
                  Steven L. Brunton},
  title        = {{Randomized Dynamic Mode Decomposition}},
  journal      = {{SIAM} J. Appl. Dyn. Syst.},
  volume       = {18},
  number       = {4},
  pages        = {1867--1891},
  year         = {2019}
}

@inproceedings{he2016resnet,
  author       = {Kaiming He and
                  Xiangyu Zhang and
                  Shaoqing Ren and
                  Jian Sun},
  title        = {{Deep Residual Learning for Image Recognition}},
  booktitle    = {{CVPR}},
  pages        = {770--778},
  year         = {2016}
}

@inproceedings{pascanu2013difficulty,
  author       = {Razvan Pascanu and
                  Tom{\'{a}}s Mikolov and
                  Yoshua Bengio},
  title        = {{On the difficulty of training recurrent neural networks}},
  booktitle    = {{ICML} {(3)}},
  volume       = {28},
  pages        = {1310--1318},
  year         = {2013}
}

@inproceedings{poole2016chaos,
  author       = {Ben Poole and
                  Subhaneil Lahiri and
                  Maithra Raghu and
                  Jascha Sohl{-}Dickstein and
                  Surya Ganguli},
  title        = {{Exponential expressivity in deep neural networks through transient chaos}},
  booktitle    = {{NIPS}},
  pages        = {3360--3368},
  year         = {2016}
}

@inproceedings{pennington2017dynamical,
  author       = {Jeffrey Pennington and
                  Samuel S. Schoenholz and
                  Surya Ganguli},
  title        = {{Resurrecting the sigmoid in deep learning through dynamical isometry: theory and practice}},
  booktitle    = {{NIPS}},
  pages        = {4785--4795},
  year         = {2017}
}

@inproceedings{xiong2020layernorm,
  author       = {Ruibin Xiong and
                  Yunchang Yang and
                  Di He and
                  Kai Zheng and
                  Shuxin Zheng and
                  Chen Xing and
                  Huishuai Zhang and
                  Yanyan Lan and
                  Liwei Wang and
                  Tie{-}Yan Liu},
  title        = {{On Layer Normalization in the Transformer Architecture}},
  booktitle    = {{ICML}},
  volume       = {119},
  pages        = {10524--10533},
  year         = {2020}
}

@inproceedings{miyato2018spectralnorm,
  author       = {Takeru Miyato and
                  Toshiki Kataoka and
                  Masanori Koyama and
                  Yuichi Yoshida},
  title        = {{Spectral Normalization for Generative Adversarial Networks}},
  booktitle    = {{ICLR}},
  year         = {2018}
}

@inproceedings{salimans2016weightnorm,
  author       = {Tim Salimans and
                  Diederik P. Kingma},
  title        = {{Weight Normalization: {A} Simple Reparameterization to Accelerate Training of Deep Neural Networks}},
  booktitle    = {{NIPS}},
  pages        = {901},
  year         = {2016}
}

@article{you2019lamb,
  author       = {Yang You and
                  Jing Li and
                  Sashank J. Reddi and
                  Jonathan Hseu and
                  Sanjiv Kumar and
                  Srinadh Bhojanapalli and
                  Xiaodan Song and
                  James Demmel and
                  Kurt Keutzer and
                  Cho{-}Jui Hsieh},
  title        = {{Large Batch Optimization for Deep Learning: Training {BERT} in 76 minutes}},
  booktitle    = {{ICLR}},
  year         = {2020}
}

@article{chen2023lion,
  author       = {Xiangning Chen and
                  Chen Liang and
                  Da Huang and
                  Esteban Real and
                  Kaiyuan Wang and
                  Hieu Pham and
                  Xuanyi Dong and
                  Thang Luong and
                  Cho{-}Jui Hsieh and
                  Yifeng Lu and
                  Quoc V. Le},
  title        = {{Symbolic Discovery of Optimization Algorithms}},
  booktitle    = {NeurIPS},
  year         = {2023}
}

@article{haber2017stable,
  author       = {Eldad Haber and
                  Lars Ruthotto},
  title        = {{Stable Architectures for Deep Neural Networks}},
  journal      = {CoRR},
  volume       = {abs/1705.03341},
  year         = {2017}
}

@inproceedings{chen2018neuralode,
  author       = {Tian Qi Chen and
                  Yulia Rubanova and
                  Jesse Bettencourt and
                  David Duvenaud},
  title        = {{Neural Ordinary Differential Equations}},
  booktitle    = {NeurIPS},
  pages        = {6572--6583},
  year         = {2018}
}

@book{kutz2016dmdbook,
  author       = {J. Nathan Kutz and
                  Steven L. Brunton and
                  Bingni W. Brunton and
                  Joshua L. Proctor},
  title        = {{Dynamic mode decomposition - data-driven modeling of complex systems}},
  year         = {2016}
}

@article{korda2018linear,
  author       = {Milan Korda and
                  Igor Mezic},
  title        = {{Linear predictors for nonlinear dynamical systems: Koopman operator meets model predictive control}},
  journal      = {Autom.},
  volume       = {93},
  pages        = {149--160},
  year         = {2018}
}

@inproceedings{zhang2019fixup,
  author       = {Hongyi Zhang and
                  Yann N. Dauphin and
                  Tengyu Ma},
  title        = {{Fixup Initialization: Residual Learning Without Normalization}},
  booktitle    = {{ICLR}},
  year         = {2019}
}

@inproceedings{bachlechner2020rezero,
  author       = {Thomas Bachlechner and
                  Bodhisattwa Prasad Majumder and
                  Huanru Henry Mao and
                  Gary Cottrell and
                  Julian J. McAuley},
  title        = {{ReZero is all you need: fast convergence at large depth}},
  booktitle    = {{UAI}},
  volume       = {161},
  pages        = {1352--1361},
  year         = {2021}
}

@inproceedings{brock2021nfnet,
  author       = {Andy Brock and
                  Soham De and
                  Samuel L. Smith and
                  Karen Simonyan},
  title        = {{High-Performance Large-Scale Image Recognition Without Normalization}},
  booktitle    = {{ICML}},
  volume       = {139},
  pages        = {1059--1071},
  year         = {2021}
}

@article{tropp2012tail,
  author       = {Joel A. Tropp},
  title        = {{User-Friendly Tail Bounds for Sums of Random Matrices}},
  journal      = {Found. Comput. Math.},
  volume       = {12},
  number       = {4},
  pages        = {389--434},
  year         = {2012}
}

@book{higham2008functions,
  author       = {Nicholas J. Higham},
  title        = {{Functions of matrices - theory and computation}},
  year         = {2008}
}

@phdthesis{tu2014dmd,
  title={{Dynamic mode decomposition: Theory and applications}},
  author={Tu, Jonathan H},
  year={2013},
  school={Princeton University}
}

@article{radford2019language,
  title={{Language models are unsupervised multitask learners}},
  author={Radford, Alec and Wu, Jeffrey and Child, Rewon and Luan, David and Amodei, Dario and Sutskever, Ilya and others},
  journal={OpenAI blog},
  volume={1},
  number={8},
  pages={9},
  year={2019}
}

@techreport{krizhevsky2009cifar,
  title={{Learning multiple layers of features from tiny images}},
  author={Krizhevsky, Alex and Hinton, Geoffrey and others},
  year={2009},
}

@article{kessy2018whitening,
  title={{Optimal whitening and decorrelation}},
  author={Kessy, Agnan and Lewin, Alex and Strimmer, Korbinian},
  journal={The American Statistician},
  volume={72},
  number={4},
  pages={309--314},
  year={2018},
}

@article{trefethen2005spectra,
  title={{Spectra and pseudospectra: the behavior of nonnormal matrices and operators}},
  author={Trefethen, Lloyd N and Embree, Mark},
  year={2020},
}

@article{kreiss1962stability,
  title={{\"{U}ber die Stabilit{\"a}tsdefinition f{\"u}r Differenzengleichungen die partielle Differentialgleichungen approximieren}},
  author={Kreiss, Heinz-Otto},
  journal={BIT Numerical Mathematics},
  volume={2},
  number={3},
  pages={153--181},
  year={1962},
}

@article{koopman1931hamiltonian,
  title={{Hamiltonian systems and transformation in Hilbert space}},
  author={Koopman, Bernard O},
  journal={Proceedings of the National Academy of Sciences},
  volume={17},
  number={5},
  pages={315--318},
  year={1931}
}

@article{mezic2005spectral,
  title={{Spectral properties of dynamical systems, model reduction and decompositions}},
  author={Mezi{\'c}, Igor},
  journal={Nonlinear Dynamics},
  volume={41},
  number={1},
  pages={309--325},
  year={2005},
}

@article{schmid2010dmd,
  title        = {{Dynamic mode decomposition of numerical and experimental data}},
  author       = {Schmid, Peter J},
  journal      = {Journal of Fluid Mechanics},
  volume       = {656},
  pages        = {5--28},
  year         = {2010}
}

@article{rowley2009spectral,
  title={{Spectral analysis of nonlinear flows}},
  author={Rowley, Clarence W and Mezi{\'c}, Igor and Bagheri, Shervin and Schlatter, Philipp and Henningson, Dan S},
  journal={Journal of fluid mechanics},
  volume={641},
  pages={115--127},
  year={2009},
}

@article{mezic2013analysis,
  title={{Analysis of fluid flows via spectral properties of the Koopman operator}},
  author={Mezi{\'c}, Igor},
  journal={Annual review of fluid mechanics},
  volume={45},
  number={1},
  pages={357--378},
  year={2013},
}

@misc{gokaslan2020openwebtext,
  title        = {{OpenWebText Corpus}},
  author       = {Aaron Gokaslan and Vanya Cohen},
  howpublished = {\url{http://Skylion007.github.io/OpenWebTextCorpus}},
  year         = {2019}
}

@inproceedings{fu2023hungry,
  author       = {Daniel Y. Fu and
                  Tri Dao and
                  Khaled Kamal Saab and
                  Armin W. Thomas and
                  Atri Rudra and
                  Christopher R{\'{e}}},
  title        = {{Hungry Hungry Hippos: Towards Language Modeling with State Space Models}},
  booktitle    = {{ICLR}},
  year         = {2023}
}

@inproceedings{arora2024zoology,
  author       = {Simran Arora and
                  Sabri Eyuboglu and
                  Aman Timalsina and
                  Isys Johnson and
                  Michael Poli and
                  James Zou and
                  Atri Rudra and
                  Christopher R{\'{e}}},
  title        = {Zoology: Measuring and Improving Recall in Efficient Language Models},
  booktitle    = {{ICLR}},
  year         = {2024}
}

@inproceedings{brown2020language,
  author       = {Tom B. Brown and
                  Benjamin Mann and
                  Nick Ryder and
                  Melanie Subbiah and
                  Jared Kaplan and
                  Prafulla Dhariwal and
                  Arvind Neelakantan and
                  Pranav Shyam and
                  Girish Sastry and
                  Amanda Askell and
                  Sandhini Agarwal and
                  Ariel Herbert{-}Voss and
                  Gretchen Krueger and
                  Tom Henighan and
                  Rewon Child and
                  Aditya Ramesh and
                  Daniel M. Ziegler and
                  Jeffrey Wu and
                  Clemens Winter and
                  Christopher Hesse and
                  Mark Chen and
                  Eric Sigler and
                  Mateusz Litwin and
                  Scott Gray and
                  Benjamin Chess and
                  Jack Clark and
                  Christopher Berner and
                  Sam McCandlish and
                  Alec Radford and
                  Ilya Sutskever and
                  Dario Amodei},
  title        = {Language Models are Few-Shot Learners},
  booktitle    = {NeurIPS},
  year         = {2020}
}

@article{elhage2021mathematical,
  title={A mathematical framework for transformer circuits},
  author={Elhage, Nelson and Nanda, Neel and Olsson, Catherine and Henighan, Tom and Joseph, Nicholas and Mann, Ben and Askell, Amanda and Bai, Yuntao and Chen, Anna and Conerly, Tom and others},
  journal={Transformer Circuits Thread},
  volume={1},
  number={1},
  pages={12},
  year={2021}
}

\newpage

\onecolumn

\title{Residual Koopman Spectral Profiling for Predicting and Preventing Transformer Training Instability\\(Supplementary Material)}
\maketitle
\appendix
\startcontents[apx]
\section*{Appendix Table of Contents}
\printcontents[apx]{l}{1}{\setcounter{tocdepth}{2}}

\newpage
\section{List of Notation}
\label{app:notation}
\small
\begin{longtable}{p{0.26\textwidth}p{0.69\textwidth}}
	\caption{List of notations used in the main paper and supplementary material.}                                                                       \\
	\toprule
	Symbol                                             & Meaning                                                                                         \\
	\midrule
	\endfirsthead
	\toprule
	Symbol                                             & Meaning                                                                                         \\
	\midrule
	\endhead
	\midrule
	\multicolumn{2}{r}{continued on next page}                                                                                                           \\
	\endfoot
	\bottomrule
	\endlastfoot

	\multicolumn{2}{l}{Core sizes and indices}                                                                                                           \\
	$L$                                                & The number of layers.                                                                           \\
	$\ell$                                             & Layer index.                                                                                    \\
	$d$                                                & Hidden dimension.                                                                               \\
	$N$                                                & The number of snapshots.                                                                        \\
	$r$                                                & Randomized DMD rank, the number of eigenvalues used in KSS.                                     \\
	\addlinespace

	\multicolumn{2}{l}{Dynamics and operators}                                                                                                           \\
	$\mathbf{h}_\ell$                                  & Residual stream at layer $\ell$.                                                                \\
	$F_\ell(\cdot)$                                    & Residual mapping $F_\ell(\mathbf{h}_\ell)=\mathbf{h}_\ell+f_\ell(\mathbf{h}_\ell;\theta_\ell)$. \\
	$\mathcal{K}$                                      & Koopman operator.                                                                               \\
	$\hat{\mathbf{A}}_\ell$                            & DMD estimate of the layer-$\ell$ Koopman operator.                                              \\
	$\mathbf{V}, \boldsymbol{\Lambda}$                 & Eigenvectors and eigenvalues of $\hat{\mathbf{A}}_\ell$.                                        \\
	$\lambda_j$                                        & Eigenvalue.                                                                                     \\
	$\rho(\mathbf{A})$                                 & Spectral radius.                                                                                \\
	$\kappa(\mathbf{V})$                               & Eigenvector condition number, a measure of non-normality.                                       \\
	$\mathcal{K}(\mathbf{A})$                          & Kreiss constant.                                                                                \\
	\addlinespace

	\multicolumn{2}{l}{Snapshots and whitening}                                                                                                          \\
	$\mathbf{X}_\ell, \mathbf{Y}_\ell$                 & Layer-$\ell$ snapshot matrices.                                                                 \\
	$\tilde{\mathbf{X}}_\ell, \tilde{\mathbf{Y}}_\ell$ & Whitened snapshots.                                                                             \\
	$\hat{\boldsymbol{\Sigma}}_X$                      & Regularized sample covariance used for whitening.                                               \\
	$(\cdot)^\dagger$                                  & Moore--Penrose pseudoinverse.                                                                   \\
	\addlinespace

	\multicolumn{2}{l}{Spectral diagnostics}                                                                                                             \\
	$M_{>1}$                                           & Unstable spectral mass.                                                                         \\
	$M_{\approx 1}$                                    & Near-unit spectral mass.                                                                        \\
	$M_{<1}$                                           & Contractive spectral mass.                                                                      \\
	$\eta_{\mathrm{nl}}$                               & Nonlinearity ratio, a fit-error measure.                                                        \\
	$\epsilon_u, \epsilon_n, \delta_c$                 & Thresholds defining spectral-mass bins.                                                         \\
	\addlinespace

	\multicolumn{2}{l}{KSS regularization}                                                                                                               \\
	$\mathcal{L}_{\mathrm{KSS}}^\ell$                  & KSS loss for layer $\ell$.                                                                      \\
	$\alpha$                                           & KSS regularization weight.                                                                      \\
	$\tau_u, \tau_l$                                   & Upper and lower band thresholds in KSS.                                                         \\
	$\gamma$                                           & Target near-unit mass level.                                                                    \\
	$m_\ell^{\mathrm{soft}}$                           & Soft near-unit mass estimate.                                                                   \\
	\addlinespace

	\multicolumn{2}{l}{Probability and norms}                                                                                                            \\
	$D$                                                & Divergence indicator.                                                                           \\
	$P(D=1\mid\mathcal{S})$                            & Predicted divergence probability.                                                               \\
	$\E[\cdot]$                                        & Expectation.                                                                                    \\
	$\Tr(\cdot)$                                       & Trace.                                                                                          \\
	$\norm{\cdot}_2, \norm{\cdot}_F$                   & Operator and Frobenius norms.                                                                   \\
	$\abs{\cdot}$                                      & Absolute value.                                                                                 \\
	$\R, \C$                                           & Real and complex number fields.                                                                 \\
\end{longtable}
\normalsize

\section{Additional Theoretical Results}
\label{app:theory}

The following results extend the theoretical analysis presented in Section~\ref{sec:theory}.

\subsection{Supplement to Theorem~\ref{thm:near_unit}}

\begin{proposition}[Bauer--Fike: Non-normality Caveat]
	\label{prop:bauer_fike}
	Let $\mathbf{A} \in \C^{d \times d}$ be diagonalizable with eigendecomposition $\mathbf{A} = \mathbf{V}\boldsymbol{\Lambda}\mathbf{V}^{-1}$. For any perturbation $\mathbf{E}$ with $\norm{\mathbf{E}}_2 \leq \delta$ and any eigenvalue $\tilde{\lambda} \in \spec(\mathbf{A}+\mathbf{E})$, we have
	\begin{equation}
		\min_{k} |\tilde{\lambda} - \lambda_k| \leq \kappa(\mathbf{V}) \cdot \delta.
		\label{eq:bauer_fike}
	\end{equation}
\end{proposition}

\begin{proof}[Proof]
	Let $\tilde{\lambda}\in \spec(\mathbf{A}+\mathbf{E})$ with eigenvector $\mathbf{x}\neq \mathbf{0}$, that is, $(\mathbf{A}+\mathbf{E})\mathbf{x}=\tilde{\lambda}\mathbf{x}$.
	Write $\mathbf{A}=\mathbf{V}\boldsymbol{\Lambda}\mathbf{V}^{-1}$ and set $\mathbf{y}\triangleq \mathbf{V}^{-1}\mathbf{x}\neq \mathbf{0}$.
	Left-multiplying by $\mathbf{V}^{-1}$ gives
	\[
		(\boldsymbol{\Lambda}-\tilde{\lambda}\mathbf{I})\mathbf{y} = -\mathbf{V}^{-1}\mathbf{E}\mathbf{V}\mathbf{y}.
	\]
	Taking Euclidean norms,
	\[
		\norm{(\boldsymbol{\Lambda}-\tilde{\lambda}\mathbf{I})\mathbf{y}}_2
		\le \norm{\mathbf{V}^{-1}}_2\norm{\mathbf{E}}_2\norm{\mathbf{V}}_2\norm{\mathbf{y}}_2
		= \kappa(\mathbf{V})\delta\norm{\mathbf{y}}_2.
	\]
	Because $\boldsymbol{\Lambda}-\tilde{\lambda}\mathbf{I}$ is diagonal with diagonal entries $(\lambda_j-\tilde{\lambda})$,
	\[
		\norm{(\boldsymbol{\Lambda}-\tilde{\lambda}\mathbf{I})\mathbf{y}}_2
		\ge
		\min_k |\lambda_k-\tilde{\lambda}| \cdot \norm{\mathbf{y}}_2.
	\]
	Canceling $\norm{\mathbf{y}}_2$ yields \eqref{eq:bauer_fike}.
\end{proof}

\subsection{DMD Convergence Analysis}

We first establish finite-sample convergence guarantees for DMD estimation in the presence of nonlinearity.

\begin{assumption}[Data Distribution]
	\label{ass:data}
	Let $(\mathbf{x}, \mathbf{y}) \in \R^d \times \R^d$ be a random pair drawn from the joint distribution induced by a layer transition. We assume centered covariates: $\E[\mathbf{x}] = \mathbf{0}$ and $\E[\mathbf{x}\mathbf{x}^\top] = \boldsymbol{\Sigma}$ with $\sigma_{\min}(\boldsymbol{\Sigma}) \geq \sigma_0 > 0$. We also assume sub-Gaussian tails: $\norm{\mathbf{x}}_{\psi_2} \leq K$ and $\norm{\mathbf{y}}_{\psi_2} \leq K$ for some $K > 0$.
	Define the cross-covariance $\mathbf{C}_{yx} \triangleq \E[\mathbf{y}\mathbf{x}^\top]$. Let $\epsilon \ge 0$ be the whitening regularizer in \eqref{eq:whitening}, and set $\boldsymbol{\Sigma}_\epsilon \triangleq \boldsymbol{\Sigma} + \epsilon\mathbf{I}$. Then $\sigma_{\min}(\boldsymbol{\Sigma}_\epsilon)\ge \sigma_0$.
	Define
	\begin{align}
		\mathbf{M}_\epsilon & \triangleq \boldsymbol{\Sigma}_\epsilon^{-1/2}\mathbf{C}_{yx}\boldsymbol{\Sigma}_\epsilon^{-1/2},     \\
		\mathbf{G}_\epsilon & \triangleq \boldsymbol{\Sigma}_\epsilon^{-1/2}\boldsymbol{\Sigma}\boldsymbol{\Sigma}_\epsilon^{-1/2},
	\end{align}
	and the regularized whitened population least-squares operator
	\begin{equation}
		\mathbf{A}_{\mathrm{w},\epsilon}^{\mathrm{LS}} \triangleq \mathbf{M}_\epsilon\mathbf{G}_\epsilon^{-1}.
	\end{equation}
\end{assumption}

The nonlinearity ratio $\eta_{\mathrm{nl}}$ is a normalized linear-fit error that we use as a practical diagnostic for linear-approximation reliability.

\begin{theorem}[Whitened DMD Finite-Sample Convergence]
	\label{thm:dmd_error}
	Under Assumption~\ref{ass:data}, let $\sigma_\epsilon \triangleq \sigma_{\min}(\boldsymbol{\Sigma}_\epsilon)$. Then $\sigma_\epsilon \ge \sigma_0$.
	Let $\delta_{\mathrm{fail}}\in(0,1)$. Consider the events
	\begin{equation}
		\label{eq:whitening_condition}
		\norm{\hat{\boldsymbol{\Sigma}}_X - \boldsymbol{\Sigma}_\epsilon}_2 \le \frac{1}{2}\sigma_\epsilon,
	\end{equation}
	and
	\begin{equation}
		\label{eq:G_stability_event}
		\norm{\hat{\mathbf{G}} - \mathbf{G}_\epsilon}_2 \le \frac{1}{2}\sigma_{\min}(\mathbf{G}_\epsilon),
	\end{equation}
	where $\hat{\mathbf{G}} \triangleq \hat{\boldsymbol{\Sigma}}_X^{-1/2}\hat{\boldsymbol{\Sigma}}\hat{\boldsymbol{\Sigma}}_X^{-1/2}$. Then $\hat{\mathbf{G}}$ is invertible, and $\norm{\hat{\mathbf{G}}^{-1}}_2 \le 2\norm{\mathbf{G}_\epsilon^{-1}}_2$.
	Assuming $\tilde{\mathbf{X}}\tilde{\mathbf{X}}^\top$ is invertible, for example, when $N \ge d$ and $\mathrm{rank}(\tilde{\mathbf{X}})=d$, there exist absolute constants $C,c>0$ such that if
	$N \ge c(d+\log(2/\delta_{\mathrm{fail}}))$, then on the event \eqref{eq:whitening_condition} and \eqref{eq:G_stability_event}:
	\begin{equation}
		\norm{\hat{\mathbf{A}} - \mathbf{A}_{\mathrm{w},\epsilon}^{\mathrm{LS}}}_2
		\leq
		C\norm{\mathbf{G}_\epsilon^{-1}}_2\frac{K^2}{\sigma_\epsilon}\Delta_N
		+ C\norm{\mathbf{G}_\epsilon^{-1}}_2^{2}\frac{K^2\norm{\mathbf{C}_{yx}}_2}{\sigma_\epsilon^{2}}\Delta_N,
		\label{eq:dmd_bound}
	\end{equation}
	where $\Delta_N \triangleq \sqrt{\frac{d + \log(2/\delta_{\mathrm{fail}})}{N}}+\frac{d + \log(2/\delta_{\mathrm{fail}})}{N}$.
\end{theorem}

\begin{proof}[Proof]
	We bound the estimation error relative to the whitened population least-squares operator $\mathbf{A}_{\mathrm{w},\epsilon}^{\mathrm{LS}}$.
	Assume $\tilde{\mathbf{X}}\tilde{\mathbf{X}}^\top$ is invertible, so that $\tilde{\mathbf{X}}^\dagger = \tilde{\mathbf{X}}^\top(\tilde{\mathbf{X}}\tilde{\mathbf{X}}^\top)^{-1}$ and
	$\hat{\mathbf{A}} = \hat{\mathbf{M}}\hat{\mathbf{G}}^{-1}$ with
	$\hat{\mathbf{M}} \triangleq \hat{\boldsymbol{\Sigma}}_X^{-1/2}\hat{\mathbf{C}}_{yx}\hat{\boldsymbol{\Sigma}}_X^{-1/2}$ and
	$\hat{\mathbf{G}} \triangleq \hat{\boldsymbol{\Sigma}}_X^{-1/2}\hat{\boldsymbol{\Sigma}}\hat{\boldsymbol{\Sigma}}_X^{-1/2}$, where
	$\bar{\mathbf{x}} = \frac{1}{N}\sum_{i=1}^N \mathbf{x}_i$, $\bar{\mathbf{y}} = \frac{1}{N}\sum_{i=1}^N \mathbf{y}_i$,
	$\hat{\boldsymbol{\Sigma}} = \frac{1}{N-1}\sum_{i=1}^N (\mathbf{x}_i-\bar{\mathbf{x}})(\mathbf{x}_i-\bar{\mathbf{x}})^\top$,
	$\hat{\mathbf{C}}_{yx} = \frac{1}{N-1}\sum_{i=1}^N (\mathbf{y}_i-\bar{\mathbf{y}})(\mathbf{x}_i-\bar{\mathbf{x}})^\top$, and
	$\hat{\boldsymbol{\Sigma}}_X = \hat{\boldsymbol{\Sigma}} + \epsilon\mathbf{I}$.

	First, we estimate the covariance. For $N$ independently and identically distributed samples with $\norm{\mathbf{x}}_{\psi_2} \leq K$, standard covariance concentration for the centered sample covariance yields \citep{tropp2012tail}
	\begin{equation}
		\norm{\hat{\boldsymbol{\Sigma}}_X - \boldsymbol{\Sigma}_\epsilon}_2 \leq C K^2 \Delta_N,
	\end{equation}
	with probability $\geq 1-\delta_{\mathrm{fail}}/2$ for $N \gtrsim d+\log(1/\delta_{\mathrm{fail}})$.

	Second, we bound the whitening perturbation. Standard perturbation theory for matrix square roots gives \citep{higham2008functions}:
	\begin{equation}
		\norm{\hat{\boldsymbol{\Sigma}}_X^{-1/2} - \boldsymbol{\Sigma}_\epsilon^{-1/2}}_2 \leq \frac{2}{\sigma_\epsilon^{3/2}} \norm{\hat{\boldsymbol{\Sigma}}_X - \boldsymbol{\Sigma}_\epsilon}_2,
	\end{equation}
	for $\norm{\hat{\boldsymbol{\Sigma}}_X - \boldsymbol{\Sigma}_\epsilon}_2 \leq \frac{1}{2}\sigma_\epsilon$.

	Now, we combine the cross-covariance and covariance estimation errors. Let
	$\hat{\mathbf{C}}_{yx} = \frac{1}{N-1}\sum_{i=1}^N (\mathbf{y}_i-\bar{\mathbf{y}})(\mathbf{x}_i-\bar{\mathbf{x}})^\top$.
	A similar sub-exponential matrix concentration bound gives \citep{tropp2012tail}
	$\norm{\hat{\mathbf{C}}_{yx}-\mathbf{C}_{yx}}_2 \le C K^2 \Delta_N$ with probability $\ge 1-\delta_{\mathrm{fail}}/2$.
	Decomposing
	\[
		\hat{\mathbf{A}} - \mathbf{A}_{\mathrm{w},\epsilon}^{\mathrm{LS}}
		=
		\hat{\mathbf{M}}\hat{\mathbf{G}}^{-1}
		-
		\mathbf{M}_\epsilon\mathbf{G}_\epsilon^{-1}
		=
		(\hat{\mathbf{M}}-\mathbf{M}_\epsilon)\mathbf{G}_\epsilon^{-1}
		+
		\hat{\mathbf{M}}\left(\hat{\mathbf{G}}^{-1}-\mathbf{G}_\epsilon^{-1}\right),
	\]
	and using a standard matrix inverse perturbation bound,
	\[
		\norm{\hat{\mathbf{G}}^{-1}-\mathbf{G}_\epsilon^{-1}}_2
		\le
		\norm{\hat{\mathbf{G}}^{-1}}_2\norm{\hat{\mathbf{G}}-\mathbf{G}_\epsilon}_2\norm{\mathbf{G}_\epsilon^{-1}}_2
		\le
		2\norm{\mathbf{G}_\epsilon^{-1}}_2^{2}\norm{\hat{\mathbf{G}}-\mathbf{G}_\epsilon}_2
	\]
	on \eqref{eq:G_stability_event} yields a cross-covariance term scaling as $\norm{\mathbf{G}_\epsilon^{-1}}_2\sigma_\epsilon^{-1}$ and a whitening term plus a covariance term scaling as $\norm{\mathbf{G}_\epsilon^{-1}}_2^{2}\norm{\mathbf{C}_{yx}}_2\sigma_\epsilon^{-2}$, giving \eqref{eq:dmd_bound}.

	Combining the bounds yields \eqref{eq:dmd_bound} under the event \eqref{eq:whitening_condition} and \eqref{eq:G_stability_event}.
\end{proof}

\begin{remark}[On $\mathbf{G}_\epsilon^{-1}$ for $\boldsymbol{\Sigma}_\epsilon = \boldsymbol{\Sigma} + \epsilon\mathbf{I}$]
	Because $\boldsymbol{\Sigma}$ and $\boldsymbol{\Sigma}_\epsilon$ commute, $\mathbf{G}_\epsilon$ has eigenvalues $\lambda_i/(\lambda_i+\epsilon)$, hence
	\[
		\norm{\mathbf{G}_\epsilon^{-1}}_2
		= \frac{\sigma_{\min}(\boldsymbol{\Sigma})+\epsilon}{\sigma_{\min}(\boldsymbol{\Sigma})}.
	\]
	If $\sigma_{\min}(\boldsymbol{\Sigma})$ is treated as a fixed constant bounded away from $0$, the factors of $\norm{\mathbf{G}_\epsilon^{-1}}_2$ can be absorbed into the constant $C$.
\end{remark}

\begin{remark}[Sample Complexity]
	Theorem~\ref{thm:dmd_error} suggests that, under the stability events \eqref{eq:whitening_condition}--\eqref{eq:G_stability_event}, $N = \tilde{O}\left(d\left(\norm{\mathbf{G}_\epsilon^{-1}}_2\frac{K^2}{\sigma_\epsilon} + \norm{\mathbf{G}_\epsilon^{-1}}_2^{2}\frac{K^2\norm{\mathbf{C}_{yx}}_2}{\sigma_\epsilon^2}\right)^2 \varepsilon^{-2}\right)$ samples are sufficient for $\varepsilon$-accurate DMD estimation, up to logarithmic factors. For typical transformers with $d = 256$ to $768$, $N \approx 2048$ provides reliable estimates.
\end{remark}

\begin{remark}[Modeling Mismatch and $\eta_{\mathrm{nl}}$]
	Theorem~\ref{thm:dmd_error} is an estimation bound for the whitened population least-squares operator $\mathbf{A}_{\mathrm{w},\epsilon}^{\mathrm{LS}}$.
	When the layer transition is nonlinear, $\mathbf{A}_{\mathrm{w},\epsilon}^{\mathrm{LS}}$ can be a poor proxy for other targets, such as a local Jacobian or a richer Koopman approximation, even if it is well-estimated.
	We use the empirical nonlinearity ratio $\eta_{\mathrm{nl}}$ defined in \eqref{eq:nonlinearity} as a practical diagnostic for when linear DMD features are less reliable.
\end{remark}

\subsection{Non-Normality and Transient Growth}

Spectral radius bounds alone are insufficient for analyzing non-normal matrices. The Kreiss matrix theorem provides tight bounds on transient behavior \citep{kreiss1962stability,trefethen2005spectra}.

\begin{theorem}[Kreiss Constant Characterization]
	\label{thm:kreiss}
	The Kreiss constant of $\mathbf{A} \in \C^{d \times d}$ is:
	\begin{equation}
		\mathcal{K}(\mathbf{A}) \triangleq \sup_{|z|>1} (|z|-1)\norm{(z\mathbf{I} - \mathbf{A})^{-1}}_2
	\end{equation}
	Assume $\mathbf{A}$ is power-bounded, that is, $\sup_{n \ge 0}\norm{\mathbf{A}^n}_2 < \infty$; equivalently, $\mathcal{K}(\mathbf{A})<\infty$.
	Then the Kreiss matrix theorem states:
	\begin{equation}
		\mathcal{K}(\mathbf{A}) \leq \sup_{n \geq 0} \norm{\mathbf{A}^n}_2 \leq e \cdot d \cdot \mathcal{K}(\mathbf{A})
		\label{eq:kreiss}
	\end{equation}
	For a diagonalizable $\mathbf{A} = \mathbf{V}\boldsymbol{\Lambda}\mathbf{V}^{-1}$:
	\begin{equation}
		\mathcal{K}(\mathbf{A}) \leq \kappa(\mathbf{V}) \cdot \sup_{|z|>1} \max_j \frac{|z|-1}{|z-\lambda_j|}
		\label{eq:kreiss_diag}
	\end{equation}
\end{theorem}

\begin{proof}[Proof]
	We first prove the two inequalities in \eqref{eq:kreiss} and then \eqref{eq:kreiss_diag}.

	\paragraph{Lower bound: $\mathcal{K}(\mathbf{A}) \le \sup_{n\ge 0}\norm{\mathbf{A}^n}_2$.}
	Let $M\triangleq \sup_{n\ge 0}\norm{\mathbf{A}^n}_2 < \infty$.
	Because $\mathbf{A}$ is power-bounded, $\sup_{n\ge 0}\norm{\mathbf{A}^n}_2 < \infty$, so the Neumann series converges in operator norm for any $|z|>1$. For any $|z|>1$, the Neumann series gives the operator-norm expansion
	\[
		(z\mathbf{I}-\mathbf{A})^{-1}
		= z^{-1}\sum_{n=0}^{\infty} \mathbf{A}^n z^{-n},
	\]
	hence
	\[
		\norm{(z\mathbf{I}-\mathbf{A})^{-1}}_2
		\le \frac{1}{|z|}\sum_{n=0}^{\infty} \frac{\norm{\mathbf{A}^n}_2}{|z|^n}
		\le \frac{M}{|z|}\sum_{n=0}^{\infty} |z|^{-n}
		= \frac{M}{|z|-1}.
	\]
	Multiplying by $(|z|-1)$ and taking the supremum over $|z|>1$ yields $\mathcal{K}(\mathbf{A})\le M$.

	\paragraph{Upper bound: $\sup_{n\ge 0}\norm{\mathbf{A}^n}_2 \le ed\mathcal{K}(\mathbf{A})$.}
	This is the finite-dimensional Kreiss matrix theorem: the resolvent bound
	$\sup_{|z|>1} (|z|-1)\norm{(z\mathbf{I}-\mathbf{A})^{-1}}_2 < \infty$
	is equivalent to power-boundedness, and quantitatively implies
	$\sup_{n\ge 0}\norm{\mathbf{A}^n}_2 \le C_d\mathcal{K}(\mathbf{A})$
	for an explicit dimension-dependent constant $C_d$; one standard choice is $C_d=ed$.

	\paragraph{Diagonalizable case.}
	If $\mathbf{A}=\mathbf{V}\boldsymbol{\Lambda}\mathbf{V}^{-1}$, then for any $z\notin \spec(\mathbf{A})$,
	\[
		(z\mathbf{I}-\mathbf{A})^{-1}
		=
		\mathbf{V}(z\mathbf{I}-\boldsymbol{\Lambda})^{-1}\mathbf{V}^{-1}.
	\]
	Taking norms gives
	\[
		\norm{(z\mathbf{I}-\mathbf{A})^{-1}}_2
		\le \norm{\mathbf{V}}_2\norm{(z\mathbf{I}-\boldsymbol{\Lambda})^{-1}}_2\norm{\mathbf{V}^{-1}}_2
		= \kappa(\mathbf{V})\norm{(z\mathbf{I}-\boldsymbol{\Lambda})^{-1}}_2.
	\]
	Because $(z\mathbf{I}-\boldsymbol{\Lambda})^{-1}$ is diagonal with diagonal entries $(z-\lambda_j)^{-1}$,
	its spectral norm is $\max_j |z-\lambda_j|^{-1}$, hence
	\[
		(|z|-1)\norm{(z\mathbf{I}-\mathbf{A})^{-1}}_2
		\le
		\kappa(\mathbf{V}) \cdot \max_j \frac{|z|-1}{|z-\lambda_j|}.
	\]
	Taking the supremum over $|z|>1$ yields \eqref{eq:kreiss_diag}.
\end{proof}

The interpretation is as follows: a high $\mathcal{K}(\mathbf{A})$ indicates hidden instability. Even when $\rho(\mathbf{A}) \leq 1$, non-orthogonal eigenvectors produce a transient growth $\norm{\mathbf{A}^n}_2 \gg 1$ for intermediate $n$.

\section{PROOF OF THEOREM \texorpdfstring{\ref{thm:near_unit}}{\ref{thm:near_unit}}}
\label{app:proofs}

\begin{proof}
	For $\mathbf{x}$ uniform on the unit sphere, rotational invariance implies $\E[\mathbf{x}\mathbf{x}^*]=\frac{1}{d}\mathbf{I}$. Therefore
	\[
		\E \norm{\mathbf{A}\mathbf{x}}_2^2
		= \E[\mathbf{x}^*\mathbf{A}^*\mathbf{A}\mathbf{x}]
		= \Tr\left(\mathbf{A}^*\mathbf{A}\E[\mathbf{x}\mathbf{x}^*]\right)
		= \frac{1}{d}\Tr(\mathbf{A}^*\mathbf{A})
		= \frac{1}{d}\norm{\mathbf{A}}_F^2.
	\]
	For a normal $\mathbf{A}$, $\norm{\mathbf{A}}_F^2=\sum_{j=1}^d |\lambda_j|^2$, giving
	\[
		\E \norm{\mathbf{A}\mathbf{x}}_2^2 = \frac{1}{d}\sum_{j=1}^d |\lambda_j|^2.
	\]
	If $\rho(\mathbf{A})\le 1+\epsilon_u$, then $|\lambda_j|^2\le (1+\epsilon_u)^2$ for all $j$, giving the upper bound. For the lower bound, at least a fraction $M_{\approx 1}(\mathbf{A})$ of the eigenvalues satisfy $|\lambda_j|\ge 1-\epsilon_n$, so
	\[
		\E \norm{\mathbf{A}\mathbf{x}}_2^2 \ge (1-\epsilon_n)^2 M_{\approx 1}(\mathbf{A}).
	\]
	For the diagonalizable extension $\mathbf{A}=\mathbf{V}\boldsymbol{\Lambda}\mathbf{V}^{-1}$, note that for an isotropic $\mathbf{x}$ we always have
	$\E \norm{\mathbf{A}\mathbf{x}}_2^2=\frac{1}{d}\norm{\mathbf{A}}_F^2$.
	Moreover,
	$\frac{1}{\kappa(\mathbf{V})}\norm{\boldsymbol{\Lambda}}_F \le \norm{\mathbf{A}}_F \le \kappa(\mathbf{V})\norm{\boldsymbol{\Lambda}}_F$,
	so $\E \norm{\mathbf{A}\mathbf{x}}_2^2$ is within factors $\kappa(\mathbf{V})^{\pm 2}$ of $\frac{1}{d}\sum_j |\lambda_j|^2$.
	Combining this with $\rho(\mathbf{A})\le 1+\epsilon_u$ and the definition of $M_{\approx 1}(\mathbf{A})$ yields the stated bound.
\end{proof}

\paragraph{Proof of Corollary~\ref{cor:near_unit_prob}.}
\begin{proof}
	Assume $\E[\mathbf{h}_\ell\mathbf{h}_\ell^*] = \frac{1}{d}\E\norm{\mathbf{h}_\ell}_2^2 \mathbf{I}$ for each $\ell=0,\dots,L-1$. Let $q_\ell \triangleq \frac{1}{d}\sum_{j=1}^d |\lambda_j^\ell|^2$. Then
	\[
		\E \norm{\mathbf{h}_{\ell+1}}_2^2
		= \Tr\left(\mathbf{A}_\ell^*\mathbf{A}_\ell\E[\mathbf{h}_\ell\mathbf{h}_\ell^*]\right)
		= q_\ell \E \norm{\mathbf{h}_\ell}_2^2,
	\]
	and recursion yields $\E\norm{\mathbf{h}_L}_2^2 = \E\norm{\mathbf{h}_0}_2^2 \prod_{\ell=0}^{L-1} q_\ell$.
	By Theorem~\ref{thm:near_unit}, $q_\ell \ge (1-\epsilon_n)^2 M_{\approx 1}(\mathbf{A}_\ell)$, proving the stated bound.
\end{proof}

\section{Experimental Details}
\label{app:details}

We ran experiments on 4$\times$ NVIDIA A100-SXM4-40GB graphics processing unit (GPU) devices. We ran experiments across six normalization strategies, including Pre-LN, Post-LN, RMSNorm, DeepNorm, SubLN, and No-Norm, with additional architecture-specific studies on MoE \citep{shazeer2017moe}, Mamba \citep{gu2023mamba}, and KAN \citep{liu2024kan}.

We use the following hyperparameters. Models use $d \in \{128, 256, 512, 768, 1024\}$, $n_{\mathrm{heads}} \in \{4, 8, 16\}$, and $L \in \{4, 6, 8, 12, 16, 24\}$. For training, the mini-batch size is 16 to 32 with 5 to 20 epochs and 100 to 1000 warmup steps. We use the AdamW and sweep learning rates within standard ranges. We use seeds $\{42, 123, 456\}$ for reproducibility. Our default recipe estimates randomized DMD eigenvalues with rank $r=32$ using $N=2048$ snapshots with $\epsilon = 10^{-5}$. KSS is applied every 10 to 20 steps, sampling 50\% of layers per update to reduce overhead. We sweep the regularization weight $\alpha \in \{0.01, 0.05, 0.10, 0.15, 0.20\}$; the resulting overhead ranges from 8\% to 12\% in practice.

\subsection{Associative-Recall Task}
\label{app:assoc_recall}
The associative-recall task refers to a synthetic key-value retrieval classification task commonly used to probe associative memory and recall in long-context sequence models \citep{fu2023hungry,arora2024zoology}. For each sample, we generate $n_{\mathrm{pairs}}$ key-value pairs $\{(k_i, v_i)\}_{i=1}^{n_{\mathrm{pairs}}}$ and construct
\[
	\mathbf{x} = [k_1, v_1, \ldots, k_{n_{\mathrm{pairs}}}, v_{n_{\mathrm{pairs}}}, p_1, \ldots, p_m, q],
\]
where $q$ is a query key chosen from $\{k_i\}$ and the label is the corresponding matched value $y=v_j$. The model predicts only this final target with cross-entropy on the last-position logits, implemented as \texttt{F.cross\_entropy(logits[:, -1, :], y)}.

\subsection{Synthetic LM Task}
\label{app:synthetic_lm}

This synthetic token-level language modeling setup is inspired by prior work using controlled synthetic sequences to analyze recall and long-range behavior in efficient sequence models \citep{fu2023hungry,arora2024zoology}.
Each sample is a length-$T$ token sequence over a vocabulary of size $V$; unless stated otherwise, we use $T=256$ and $V=10{,}000$. A sequence is generated by concatenating randomly sampled segments until reaching length $T$, then truncating. Segments come in three types: repetition segments repeat a short pattern of length 2 to 5 for 2 to 4 repeats, sequential segments are contiguous integer runs of length 5 to 15, and random segments are independently and identically distributed tokens of length 5 to 15. All tokens are sampled from $\{10,\dots,V-1\}$ so that a small identifier range remains available for special tokens. Sequences are generated directly at the token level by these rules.

The learning objective is standard next-token prediction. Given tokens $(x_1,\dots,x_T)$, the model predicts $x_{t+1}$ from the prefix $(x_1,\dots,x_t)$ and is trained with token-level cross-entropy over $t=1,\dots,T-1$. We report validation token accuracy and perplexity $\exp(\text{mean cross-entropy})$ on a held-out synthetic validation split. For large-scale runs, we use 20K training sequences and 2K validation sequences per trial, regenerated deterministically from the run seed.

\subsection{Pretrained LM Fixed-Prompt Protocol}
\label{app:pretrained_prompt_protocol}
We use a forward-only profiling protocol with a deterministic text set and no fine-tuning updates, in the same spirit as prompt-based evaluation and activation-probing analyses of pretrained transformers \citep{brown2020language,elhage2021mathematical}. The fixed short-prompt setting is implemented as explicit prompt lists with 32 total prompts per run, either 4 prompts repeated 8 times or 8 prompts repeated 4 times, and both variants use a token length cap of 64. For each model, we run a single batched forward pass with hidden-state outputs enabled and collect residual-stream activations from the embedding and transformer layer outputs.

For each layer transition $(\ell,\ell+1)$, we flatten token positions, subsample up to $N\in\{1024, 2048\}$ token states, and apply whitened DMD. Spectral partitions use the same thresholds as the analysis code: unstable when $|\lambda|>1.05$, near-unit when $0.90\le|\lambda|\le1.05$, and over-damped when $|\lambda|<0.80$. We report early, middle, and late summaries by splitting layers into depth thirds and averaging each metric within the corresponding group.

\section{Computational Cost}
Whitened DMD requires $O(d^2 N + d^3)$ operations. Randomized singular value decomposition reduces this cost to $O(dNr + r^3)$ for rank-$r$ approximation. With typical values $d = 768$, $N = 2048$, and $r = 32$, full RKSP analysis completes in 2.5 to 3.5 seconds per layer on a single GPU.

\section{Extended Baseline and Optimizer Comparisons}
\label{app:extended_baselines}

\subsection{Extended Optimizer Baselines}
\label{sec:optimizers}

We extend the baseline comparisons to include $\mu$P and the Layer-wise Adaptive Moments for Batch training (LAMB) optimizer \citep{you2019lamb}.

\paragraph{$\mu$P.} The $\mu$P enables hyperparameter transfer across different model widths by appropriately scaling learning rates appropriately. Table~\ref{tab:mup} summarizes the $\mu$P comparisons and the combined $\mu$P + KSS setting. We observe moderate stability gains from $\mu$P, with a 38\% divergence reduction through initialization scaling. The combined $\mu$P + KSS approach achieves the lowest divergence, 8.3\%, and the highest accuracy, 51.5\%. The $\mu$P setting enables transfer, while KSS provides stability, and the distinct mechanisms suggest orthogonal benefits.

\begin{table}[t]
	\centering
	\caption{Comparison of $\mu$P, standard parameterization, and KSS. Results use the No-Norm setting with learning rate ranges from 0.005 to 0.01 across 24 trials. Measured on the associative-recall task.}
	\label{tab:mup}
	\vspace{0.5em}
	\begin{tabular}{lccccc}
		\toprule
		\theadrow{\thc{Method}                      & \thc{Div.\% (lower)} & \thc{Acc.\% (higher)} & \thc{$M_{\approx 1}$} & \thc{Transfer} & \thc{Overhead}}
		\midrule
		\normalfont\unboldmath Standard Init, AdamW & $66.7$               & $28.5$                & $0.85$                & \ding{55}      & ---             \\
		$\mu$P, width transfer                      & $41.7$               & $35.8$                & $0.72$                & \checkmark     & $<1\%$          \\
		$\mu$P + higher learning rate               & $50.0$               & $38.2$                & $0.78$                & \checkmark     & $<1\%$          \\
		\midrule
		KSS, $\alpha=0.15$                          & $\mathbf{12.5}$      & $\mathbf{48.2}$       & $0.58$                & \ding{55}      & $10.8\%$        \\
		$\mu$P + KSS                                & $\mathbf{8.3}$       & $\mathbf{51.5}$       & $0.52$                & \checkmark     & $11.2\%$        \\
		\bottomrule
	\end{tabular}
\end{table}

\paragraph{LAMB Optimizer.}
LAMB normalizes updates per layer, which changes spectral dynamics. Table~\ref{tab:lamb} reports the optimizer comparison, including LAMB and Lion. LAMB outperforms AdamW: its layer-wise normalization provides implicit stability with a 31\% divergence reduction. LAMB enables higher learning rates, reaching 3$\times$ to 4$\times$ AdamW, or up to 7$\times$ with learning rate warmup. KSS still provides orthogonal benefits, and the LAMB + KSS combination achieves the best results at 8.3\% divergence and 52.8\% accuracy. From a spectral perspective, LAMB suppresses expansive mass and reduces near-unit mass, from 0.85 with AdamW to 0.75, suggesting that its stability gains come from damping unstable modes and increasing damping; within comparable unstable-mass regimes, a lower $M_{\approx 1}$ aligns with greater stability.

\begin{table}[t]
	\centering
	\caption{Optimizer comparison among AdamW, LAMB, and Lion. Results use the No-Norm setting, 24 trials each. Measured on the associative-recall task.}
	\label{tab:lamb}
	\vspace{0.5em}
	\begin{tabular}{lccccc}
		\toprule
		\theadrow{\thc{Optimizer}    & \thc{Div.\% (lower)} & \thc{Acc.\% (higher)} & \thc{$M_{\approx 1}$} & \thc{Best LR} & \thc{Overhead}}
		\midrule
		\normalfont\unboldmath AdamW & $66.7$               & $28.5$                & $0.85$                & $0.003$       & ---             \\
		AdamW + grad clip            & $58.3$               & $30.8$                & $0.83$                & $0.005$       & $<1\%$          \\
		LAMB                         & $45.8$               & $36.2$                & $0.75$                & $0.01$        & about 5\%       \\
		LAMB + warmup                & $37.5$               & $40.8$                & $0.68$                & $0.02$        & about 5\%       \\
		Lion                         & $45.8$               & $36.5$                & $0.78$                & $0.001$       & about 15\%      \\
		\midrule
		KSS, $\alpha=0.15$           & $\mathbf{12.5}$      & $\mathbf{48.2}$       & $0.58$                & $0.008$       & $10.8\%$        \\
		LAMB + KSS                   & $\mathbf{8.3}$       & $\mathbf{52.8}$       & $0.52$                & $0.015$       & about 16\%      \\
		\bottomrule
	\end{tabular}
\end{table}

\section{Large-Scale Pretrained Model Analysis}
\label{app:scale}

\paragraph{GPT-2 Analysis.} Table~\ref{tab:gpt2} reports layer-group spectral statistics for GPT-2 \citep{radford2019language}, and Figure~\ref{fig:linearity} reveals a universal pattern. The normalized linear-fit error $\eta_{\mathrm{nl}}$ increases with depth, rising from $[0.48, 0.52]$ in the early layers to $[0.68, 0.71]$ in late layers. Simultaneously, the near-unit mass decreases from $M_{\approx 1} \in [0.68, 0.72]$ to $M_{\approx 1} \in [0.58, 0.60]$. This depth-wise trend has a clear implication: early layers are more linearly approximable, making DMD features more reliable, whereas late layers are less so.

\begin{table}[t]
	\centering
	\caption{GPT-2 layer-wise spectral analysis. Start Linear, End Nonlinear pattern, shorthand for increasing $\eta_{\mathrm{nl}}$. Spectral statistics are computed from residual-stream activations on a fixed set of short prompt sentences. Measured on a fixed short-prompt set.}
	\label{tab:gpt2}
	\vspace{0.5em}
	\begin{tabular}{llcccc}
		\toprule
		\theadrow{\thc{Model}             & \thc{Layers}   & \thc{$\rho$} & \thc{$\kappa(\mathbf{V})$} & \thc{$\eta_{\mathrm{nl}}$} & \thc{$M_{\approx 1}$}}
		\midrule
		\normalfont\unboldmath GPT-2 124M & Early 0 to 3   & $1.15$       & $8.3$                      & $0.52$                     & $0.68$                 \\
		GPT-2 124M                        & Middle 4 to 7  & $1.23$       & $10.1$                     & $0.61$                     & $0.62$                 \\
		GPT-2 124M                        & Late 8 to 11   & $1.31$       & $11.2$                     & $0.71$                     & $0.58$                 \\
		\midrule
		GPT-2 355M                        & Early 0 to 7   & $1.12$       & $7.5$                      & $0.48$                     & $0.72$                 \\
		GPT-2 355M                        & Middle 8 to 15 & $1.19$       & $9.2$                      & $0.58$                     & $0.65$                 \\
		GPT-2 355M                        & Late 16 to 23  & $1.27$       & $10.8$                     & $0.68$                     & $0.60$                 \\
		\bottomrule
	\end{tabular}
\end{table}

\begin{figure}[t]
	\centering
	\includegraphics[width=1.0\columnwidth]{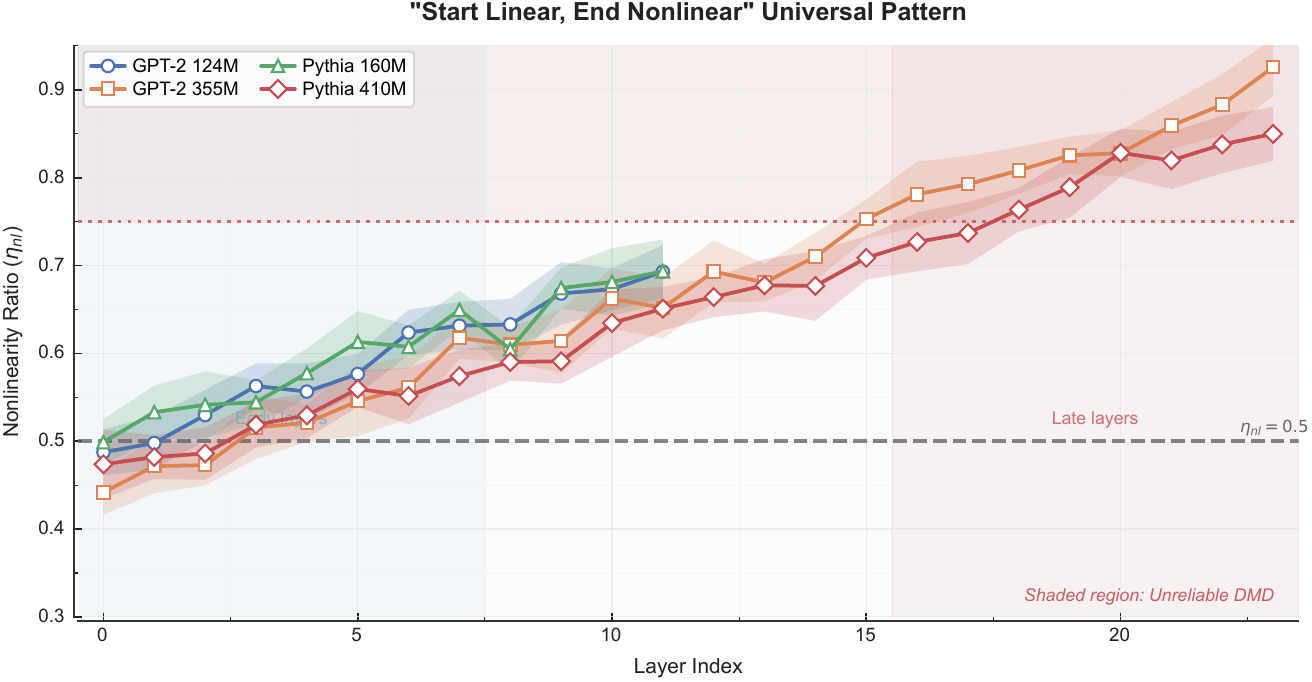}
	\caption{Start Linear, End Nonlinear pattern. Layer-wise normalized linear-fit error $\eta_{\mathrm{nl}}$ across four pretrained models. All models exhibit a monotonically increasing $\eta_{\mathrm{nl}}$ with depth, suggesting a consistent linear-approximation signature across models. Computed from residual-stream activations on a fixed set of short prompt sentences.}
	\label{fig:linearity}
\end{figure}

\section{Calibration}
\label{app:calibration}
Beyond discrimination measured by AUROC, we assess calibration quality. Figure~\ref{fig:calibration} shows that the risk score $M_{\approx 1}$ achieves an Expected Calibration Error (ECE) of 0.283, indicating moderate calibration. The reliability diagram reveals deviations between predicted probabilities and observed frequencies, while the distribution plots show clear separation between converged runs with lower $M_{\approx 1}$, corresponding to lower risk, and diverged runs with higher $M_{\approx 1}$, corresponding to higher risk. This calibration quality matters for deployment: practitioners can interpret RKSP's probability estimates for early termination decisions while accounting for the moderate calibration.

\begin{figure}[t]
	\centering
	\includegraphics[width=\columnwidth]{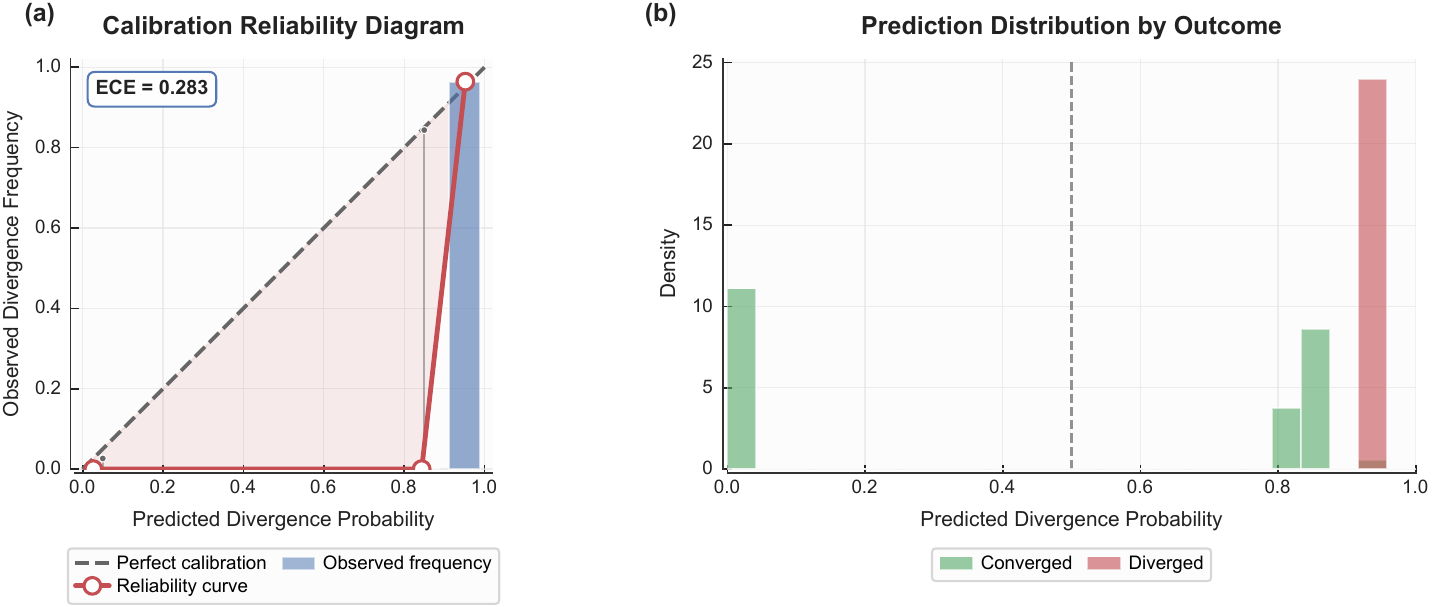}
	\caption{Calibration reliability diagram. (Left) Predicted divergence probability versus observed frequency, with an ECE of 0.283. (Right) Distribution of predictions separated by actual outcome. RKSP provides moderately calibrated probability estimates. Based on associative-recall runs, calibration compares predictions to divergence outcomes from that task.}
	\label{fig:calibration}
\end{figure}

\section{Novel Architecture Case Studies}
\label{app:novel_arch}

To demonstrate RKSP's value beyond standard transformers, we analyze three emerging architectures: MoE \citep{shazeer2017moe}, SSMs including Mamba \citep{gu2023mamba}, and KAN \citep{liu2024kan}.

\paragraph{MoE Transformers} Table~\ref{tab:moe} presents a comparison of MoE routing and stability \citep{shazeer2017moe}. MoE routing induces a higher normalized linear-fit error: $\eta_{\mathrm{nl}}$ increases 15\% to 20\% compared to dense transformers due to discrete routing decisions. The choice of top-$k$ affects spectral stability---higher $k$ shifts spectral mass and changes $M_{\approx 1}$ alongside non-normality and unstable modes. The divergence reductions are consistent with suppressing unstable modes, and within comparable non-normality regimes, larger $M_{\approx 1}$ aligns with greater instability. KSS stabilizes MoE effectively, yielding a 4$\times$ divergence reduction with 6\% accuracy improvement, thereby validating RKSP and KSS for novel architectures.

\begin{table}[t]
	\centering
	\caption{MoE transformer with RKSP analysis. Routing instability revealed via spectral signatures. Results use $d=256$, $L=6$, and 24 trials. Load balancing loss $\lambda=0.01$. Measured on the synthetic LM task with random-token next-token prediction.}
	\label{tab:moe}
	\vspace{0.5em}
	\begin{tabular}{lcccccc}
		\toprule
		\theadrow{\thc{Configuration}               & \thc{Top-$k$} & \thc{$M_{\approx 1}$} & \thc{$\rho$} & \thc{$\eta_{\mathrm{nl}}$} & \thc{Div. of 24} & \thc{Acc.\%}}
		\midrule
		\normalfont\unboldmath MoE-Small, 8 experts & $k=1$         & $0.72$                & $2.85$       & $0.68$                     & 6 of 24          & $32.4$          \\
		MoE-Small, 8 experts                        & $k=2$         & $0.58$                & $2.12$       & $0.55$                     & 2 of 24          & $41.7$          \\
		MoE-Small, 8 experts                        & $k=4$         & $0.45$                & $1.78$       & $0.48$                     & 1 of 24          & $38.2$          \\
		\midrule
		MoE-Medium, 16 experts                      & $k=2$         & $0.65$                & $2.45$       & $0.61$                     & 4 of 24          & $38.9$          \\
		MoE-Medium, 16 experts                      & $k=2$ + KSS   & $0.48$                & $1.92$       & $0.58$                     & \textbf{1 of 24} & $\mathbf{44.5}$ \\
		\bottomrule
	\end{tabular}
\end{table}

\paragraph{State Space Models: Mamba} Table~\ref{tab:mamba} compares SSM and transformer spectral properties for Mamba \citep{gu2023mamba}. The theoretical explanation is straightforward: SSMs are designed with stable discrete-time dynamics via highly structured polynomial projection operator initialization. RKSP reveals this design choice explicitly in the spectral signature: Mamba exhibited $M_{<1} \approx 0.85 \gg M_{\approx 1} \approx 0.12$. This separation indicates strongly contractive dynamics with short memory and weak near-isometry; stability here comes from suppressed unstable modes in a highly contractive regime. In transformer regimes that are closer to near-normal, larger $M_{\approx 1}$ corresponds to weaker damping, longer-range signal retention, and higher instability risk.

\begin{table}[t]
	\centering
	\caption{Mamba with RKSP analysis. Inherently stable spectral structure. Results use $L=6$, 24 trials. Measured on the synthetic LM task with random-token next-token prediction.}
	\label{tab:mamba}
	\vspace{0.5em}
	\begin{tabular}{lccccc}
		\toprule
		\theadrow{\thc{Model}                       & \thc{$M_{\approx 1}$} & \thc{$M_{<1}$} & \thc{$\rho$} & \thc{Div. of 24} & \thc{Acc.\%}}
		\midrule
		\normalfont\unboldmath Mamba-Small, $d=256$ & $0.12$                & $0.85$         & $0.95$       & 0 of 24          & $48.2$        \\
		Mamba-Medium, $d=512$                       & $0.15$                & $0.82$         & $0.97$       & 0 of 24          & $52.6$        \\
		\midrule
		Transformer, Pre-LN, comparable             & $0.42$                & $0.38$         & $1.85$       & 1 of 24          & $45.8$        \\
		\bottomrule
	\end{tabular}
\end{table}

\paragraph{KAN} Table~\ref{tab:kan} reports KAN spectral diagnostics and KSS outcomes \citep{liu2024kan}. KAN shows high normalized linear-fit error: B-spline basis functions produce $\eta_{\mathrm{nl}} \approx [0.78, 0.82]$, higher than the typical transformer layers with $[0.4, 0.7]$. Despite this high $\eta_{\mathrm{nl}}$, RKSP remains informative---ResDMD filtering enables spectral analysis for 68\% to 78\% of modes. KSS benefits KAN with a 3$\times$ to 4$\times$ divergence reduction, suggesting that spectral shaping is architecture-agnostic.

\begin{table}[t]
	\centering
	\caption{KAN transformer with RKSP analysis. The B-spline nonlinearity challenges linear approximation. The B-spline order is $B$. We use $d=256$, $L=6$, and 24 trials. Measured on the synthetic LM task with random-token next-token prediction.}
	\label{tab:kan}
	\vspace{0.5em}
	\begin{tabular}{lccccc}
		\toprule
		\theadrow{\thc{Model}                         & \thc{$\eta_{\mathrm{nl}}$} & \thc{$M_{\approx 1}$} & \thc{$\rho$} & \thc{Div. of 24} & \thc{DMD reliability}}
		\midrule
		\normalfont\unboldmath KAN-Transformer, $B=4$ & $0.78$                     & $0.52$                & $2.15$       & 3 of 24          & Marginal, 68\%         \\
		KAN-Transformer, $B=8$                        & $0.82$                     & $0.48$                & $2.35$       & 4 of 24          & Low, 52\%              \\
		\midrule
		KAN-Transformer + KSS                         & $0.75$                     & $0.42$                & $1.92$       & \textbf{1 of 24} & Improved, 78\%         \\
		\bottomrule
	\end{tabular}
\end{table}

\paragraph{Cross-Architecture Summary}
Table~\ref{tab:arch_compare} summarizes cross-architecture metrics, while Figure~\ref{fig:radar} provides a normalized radar-chart view of the same comparison.

\begin{figure}[t]
	\centering
	\includegraphics[width=0.50\columnwidth]{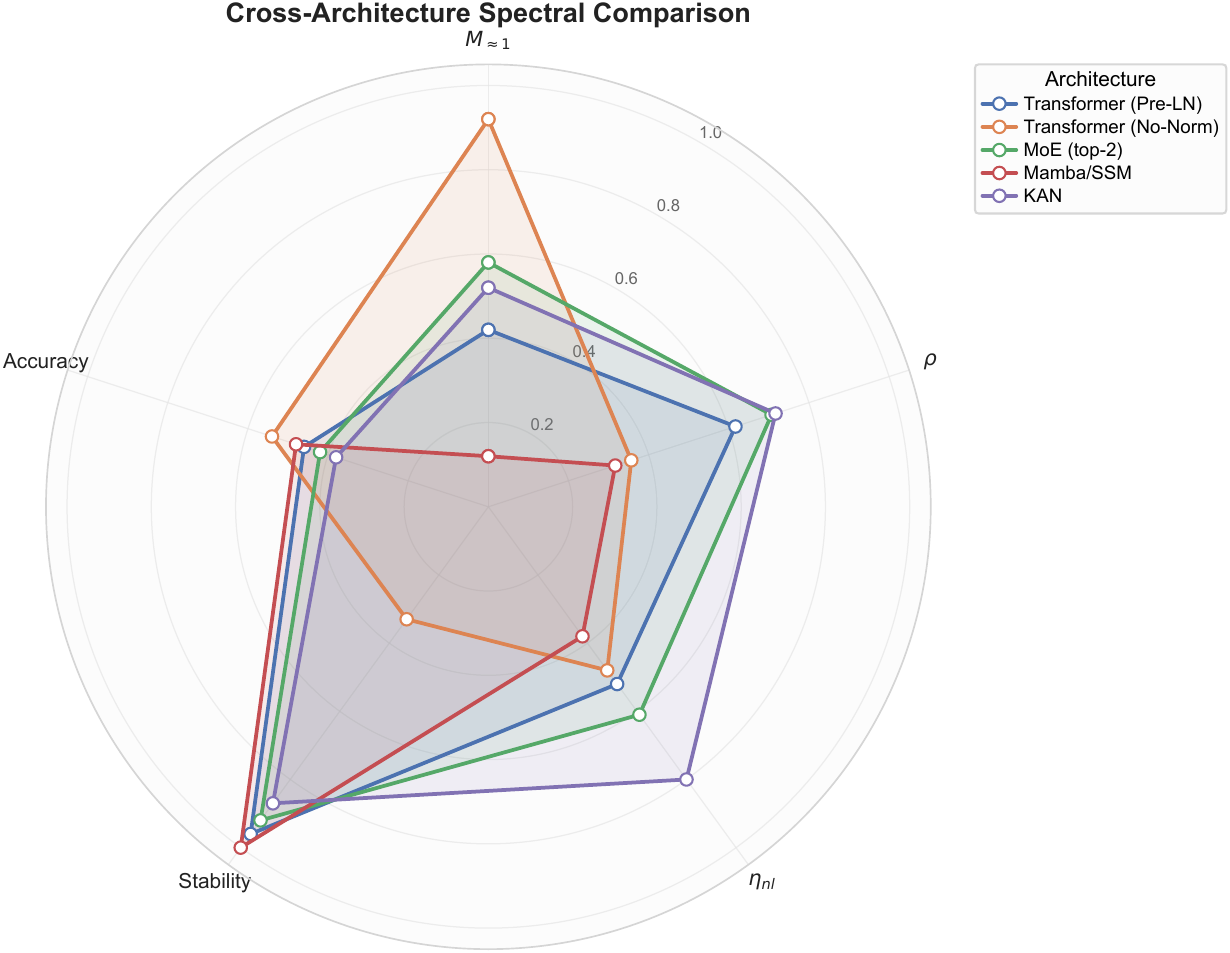}
	\caption{Cross-architecture spectral radar chart. Comparison of five architectures across five normalized metrics. Mamba exhibits strong contraction with low $M_{\approx 1}$ and short memory; stability is maintained via suppressed unstable modes, while in near-normal transformer regimes, higher $M_{\approx 1}$ aligns with more unstable, near-isometric propagation. The No-Norm transformer shows high memory capacity but poor stability. KAN exhibits high $\eta_{\mathrm{nl}}$. Metrics are derived from Table~\ref{tab:arch_compare}.}
	\label{fig:radar}
\end{figure}

\section{Practical Notes}
\label{app:practical_notes}

RKSP and KSS are most valuable in three scenarios. First, when mechanistic understanding matters, RKSP explains why Pre-LN outperforms Post-LN through spectral signatures. Second, when pushing training limits, KSS enables learning rates that are 50\% to 150\% higher for faster convergence. Third, when deploying novel architectures, RKSP verifies stability before expensive training runs. Edge cases benefit most from these diagnostics---situations where standard normalization fails or where training operates near stability boundaries.

\paragraph{Fixup and ReZero-style identity initialization.}
A common stabilization trick in deep residual networks and transformers is to initialize the final projection of each residual branch to zero, for example the attention and MLP output weights, so that the network starts close to an identity map \citep{zhang2019fixup,bachlechner2020rezero}. In our notation, this yields a residual-off regime with a vanishing layer update $\mathbf{h}_{\ell+1}-\mathbf{h}_\ell \approx \mathbf{0}$, so the snapshot pairs satisfy $\mathbf{Y}_\ell \approx \mathbf{X}_\ell$ and DMD returns $\hat{\mathbf{A}}_\ell \approx \mathbf{I}$. Consequently, $M_{\approx 1}^\ell$ can be close to $1$ across layers even though training is often stable under Fixup and ReZero at initialization.

Taken alone, a near-identity spectrum might seem to imply maximal instability risk. However, our instability mechanism assumes two conditions: weak damping with large $M_{\approx 1}$ under near-normality, and non-degenerate layer-wise dynamics with appreciable updates so that perturbations and optimization noise are repeatedly injected and propagated across depth. Fixup and ReZero violate the second condition at initialization. When $\norm{\tilde{\mathbf{Y}}_\ell-\tilde{\mathbf{X}}_\ell}_F$ is near zero, there is essentially no layer-wise update to analyze, and the resulting DMD spectrum is not informative about the noisy training-time regime we target.

Practically, this degeneracy is detectable from the same quantities RKSP already computes. When $\norm{\tilde{\mathbf{Y}}_\ell-\tilde{\mathbf{X}}_\ell}_F \approx 0$, the normalization in the nonlinearity ratio \eqref{eq:nonlinearity} becomes ill-conditioned, so $\eta_{\mathrm{nl}}(\ell)$ should be interpreted as a DMD reliability flag rather than as a meaningful nonlinearity estimate. For Fixup and ReZero, RKSP becomes informative after a small amount of training, once the zero-initialized residual projections move away from zero and layer-wise updates become observable; at that point, RKSP can again capture whether the residual stream exhibits excessive near-isometric propagation (large $M_{\approx 1}$) that correlates with high-learning-rate divergence.

Practical deployment is straightforward. We recommend using RKSP in four scenarios: first, as a fast filter during architecture search; second, before expensive hyperparameter grid search; third, for periodic spectral monitoring during training; and fourth, for debugging checkpoints before divergence. Figure~\ref{fig:decision_flowchart} provides an actionable decision process.

\begin{figure}[t]
	\centering
	\begin{tikzpicture}[
			scale=0.95,
			transform shape,
			font=\small,
			node distance=7mm,
			block/.style={rectangle, draw, rounded corners, align=left, text width=0.86\columnwidth, inner sep=5pt, fill=gray!5},
			decision/.style={diamond, draw, aspect=2.2, align=center, text width=0.46\columnwidth, inner sep=1pt, fill=gray!5},
			line/.style={-Stealth, thick, shorten >=2pt}
		]
		\node[block] (start) {Run RKSP at initialization. Compute $M_{\approx 1}$ and $\kappa(\mathbf{V})$ for the decision.};
		\node[decision, below=of start] (safe) {Meets safe-region criteria\\$M_{\approx 1} < 0.3$ and modest $\kappa(\mathbf{V})$};
		\node[block, below=of safe] (stable) {Proceed without KSS. Under this criterion, training is more stable; optionally, monitor with periodic RKSP snapshots.};
		\node[block, below=of stable] (context) {Decide on KSS by context.\\
			If $M_{\approx 1} > 0.5$ or $\kappa(\mathbf{V})$ is large, treat as high risk.\\
			Use KSS for aggressive learning rates above $2\times$ standard, for no-normalization settings, or for novel architectures with RKSP monitoring.\\
			Skip KSS for standard Pre-LN and RMSNorm at conservative learning rates.};
		\coordinate (no-x) at ([xshift=6mm]context.east);
		\coordinate (no-top) at (no-x |- safe.east);
		\coordinate (no-above) at ([yshift=4mm]context.north);
		\coordinate (no-drop) at (no-x |- no-above);
		\draw[line] (start.south) -- (safe.north);
		\draw[line] (safe.south) -- node[midway, right]{yes} (stable.north);
		\draw[line] (safe.east) -- node[midway, above]{no} (no-top) -- (no-drop) -- (no-above) -- (context.north);
	\end{tikzpicture}
	\caption{Decision flowchart for when to use RKSP and KSS in practice.}
	\label{fig:decision_flowchart}
\end{figure}
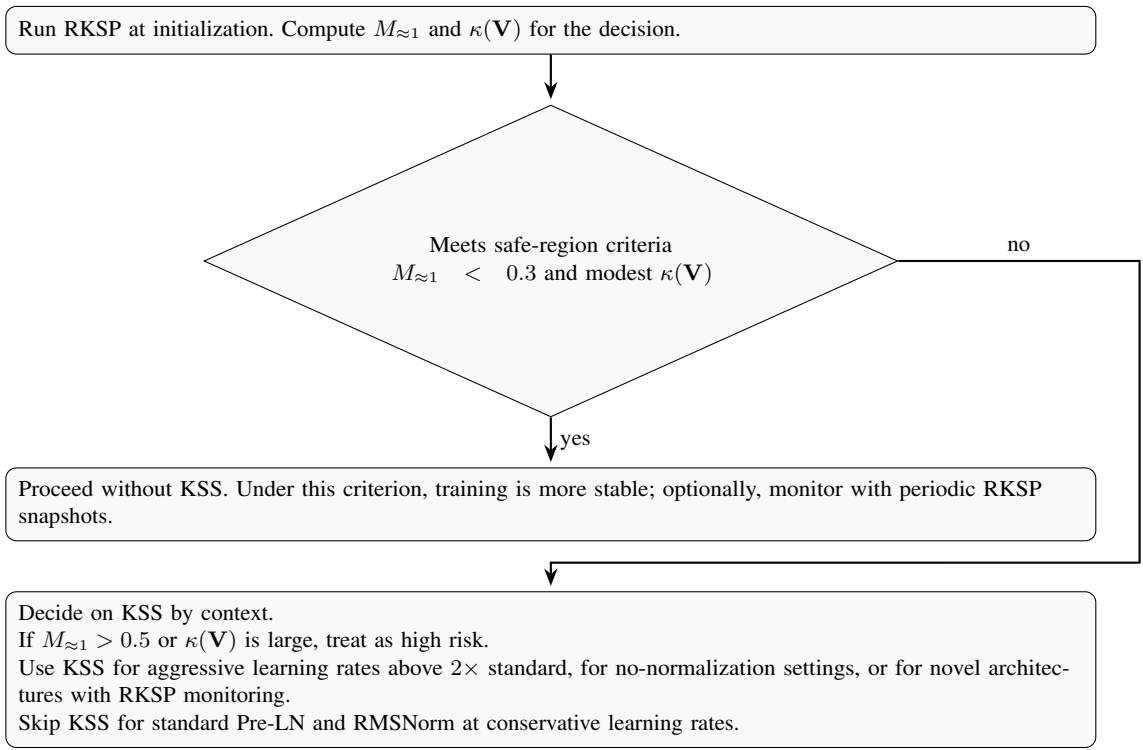

\end{document}